\documentclass{article}
\usepackage{iclr2024_conference,times}

\usepackage{amsmath,amsfonts,bm}

\def\eqref#1{equation~\ref{#1}}

\def\ceil#1{\lceil #1 \rceil}

\def\1{\bm{1}}

\def\vc{{\bm{c}}}
\def\vd{{\bm{d}}}
\def\ve{{\bm{e}}}

\def\vr{{\bm{r}}}

\def\vepsilon{{\bm{\epsilon}}}

\DeclareMathAlphabet{\mathsfit}{\encodingdefault}{\sfdefault}{m}{sl}
\SetMathAlphabet{\mathsfit}{bold}{\encodingdefault}{\sfdefault}{bx}{n}

\def\gD{{\mathcal{D}}}

\def\gL{{\mathcal{L}}}

\def\gN{{\mathcal{N}}}

\def\gX{{\mathcal{X}}}
\def\gY{{\mathcal{Y}}}

\newcommand{\E}{\mathbb{E}}

\newcommand{\R}{\mathbb{R}}

\usepackage{hyperref}
\usepackage{url}
\usepackage{graphicx}
\usepackage{amsmath}
\usepackage{amssymb}
\usepackage{amsthm}
\usepackage{booktabs}
\usepackage{multicol}
\usepackage{pifont}
\usepackage{algorithm}
\usepackage{algpseudocode}
\usepackage{ulem}
\usepackage{varwidth}
\usepackage{multirow}
\usepackage{float}
\usepackage{wrapfig}
\usepackage{enumitem}
\usepackage{pifont}

\newtheorem{definition}{Definition}

\newtheorem{lemma}{Lemma}
\newtheorem{proposition}{Proposition}

\title{\textit{Connect, Collapse, Corrupt:} Learning \\ Cross-Modal Tasks with Uni-Modal Data}

\author{Yuhui Zhang\thanks{Equal contribution. Project page available at \url{https://yuhui-zh15.github.io/C3-Website/}.}~~, Elaine Sui$^*$, Serena Yeung-Levy \\
Stanford University\\
\texttt{\{yuhuiz, esui, syyeung\}@cs.stanford.edu} 
}

\iclrfinalcopy

\begin{document}

\maketitle

\begin{abstract}
Building cross-modal applications is challenging due to limited paired multi-modal data. Recent works have shown that leveraging a pre-trained multi-modal contrastive representation space enables cross-modal tasks to be learned from uni-modal data. This is based on the assumption that contrastive optimization makes embeddings from different modalities interchangeable. However, this assumption is under-explored due to the poorly understood geometry of the multi-modal contrastive space, where a modality gap exists. In our study, we provide a theoretical explanation of this space's geometry and introduce a three-step method, $C^3$ (Connect, Collapse, Corrupt), to bridge the modality gap, enhancing the interchangeability of embeddings. Our $C^3$ method significantly improves cross-modal learning from uni-modal data, achieving state-of-the-art results on zero-shot image / audio / video captioning and text-to-image generation.
\end{abstract}

\section{Introduction}
\label{sec:introduction}

\begin{wrapfigure}{r}{0.39\linewidth}
    \centering
    \includegraphics[width=\linewidth]{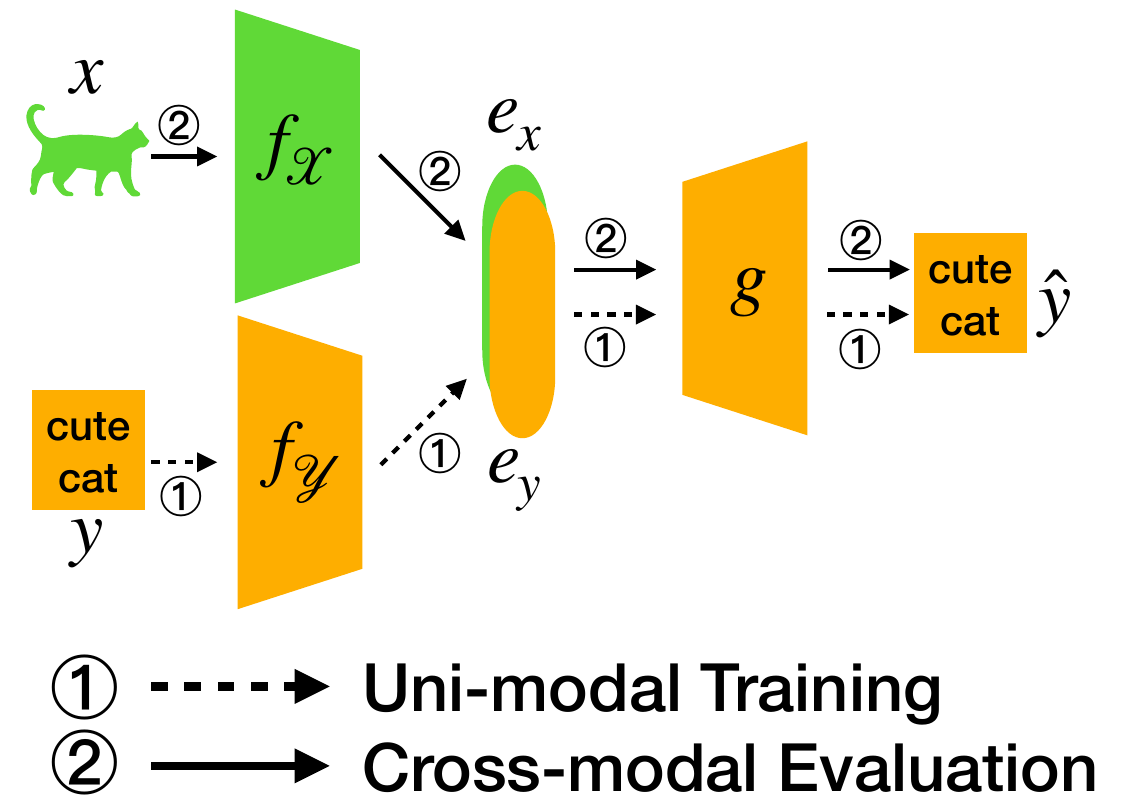}
    \vspace{-1.5em}
    \caption{Interchangeable use of embeddings enables learning cross-modal tasks with uni-modal data.}
    \label{fig:fig1}
    \vspace{-0.5em}
\end{wrapfigure}

Building cross-modal applications often face a critical challenge: the scarcity of paired multi-modal data. Large-scale multi-modal contrastive learning has emerged as a promising approach to address this gap. Pioneering works like CLIP~\citep{radford2021learning} and ImageBind~\citep{girdhar2023imagebind} offer publicly available representation spaces learned from contrasting millions of web-based images and texts. These spaces that align embeddings of concepts across different modalities pave the way for leveraging abundant uni-modal data in cross-modal tasks. For instance, instead of training an image captioning system on image embeddings, one could train it on caption embeddings. Likewise, text-to-image generation models can utilize image embeddings as opposed to text. During cross-modal inference, embeddings from the other modality are simply input into the model (Figure~\ref{fig:fig1}). Notably, this approach has achieved impressive results in image captioning~\citep{li2023decap,tam2023simple,nukrai-etal-2022-text} and text-to-image generation~\citep{zhou2022lafite,zhou2022lafite2,zhou2022shifted}, eliminating the need for multi-modal paired data.

This aforementioned process is based on the hypothesis that image embeddings and text embeddings can be used interchangeably in the multi-modal representation space.
However, the validity of this assumption is not well understood. Recent works reveal that the resulting geometry from multi-modal contrastive learning is nontrivial ---  corresponding image and text embeddings do not necessarily collapse to the same points in the space. Instead, there is a significant gap between embeddings from different modalities, potentially hindering the direct interchangeable use of image and text embeddings~\citep{liang2022mind,zhang2023diagnosing}. The lack of comprehension of the resulting geometry from contrastive learning makes it challenging to design methods to leverage this representation space and learn cross-modal tasks from uni-modal data.

\begin{wrapfigure}{r}{0.4\linewidth}
    \centering
    \includegraphics[width=\linewidth]{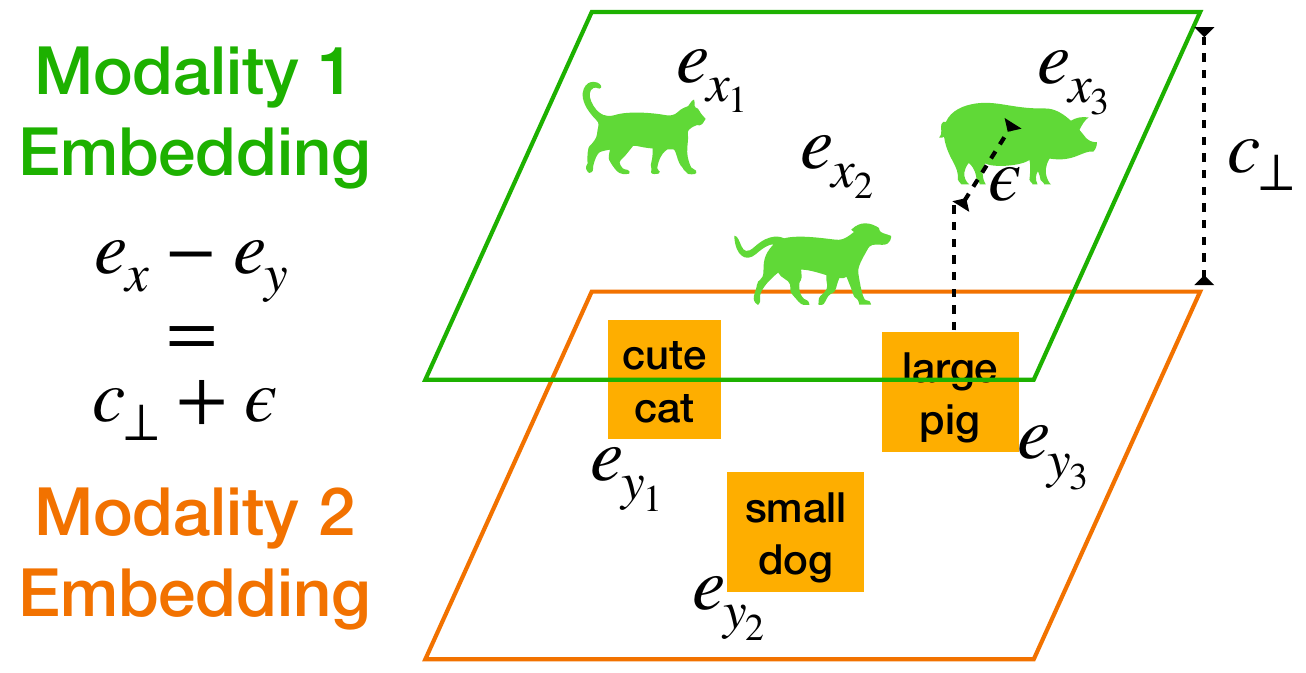}
    \vspace{-1.5em}
    \caption{Geometry of the multi-modal contrastive representation space.}
    \label{fig:fig2}
    \vspace{-0.5em}
\end{wrapfigure}

In this study, we conclude that the geometry of multi-modal contrastive representation space as (Figure~\ref{fig:fig2}):
\begin{align*}
\ve_x - \ve_y = \vc_\perp + \vepsilon
\end{align*}
where $\ve_x$ and $\ve_y$ denote the embeddings of paired inputs from different modalities, $\vc_\perp$ is a constant vector representing the \emph{modality gap}, which is orthogonal to the embedding span of $\ve_x$ and $\ve_y$, and $\epsilon$ is a random vector representing the \emph{alignment noise}, which can be approximated by a Gaussian distribution.

We provide a theoretical explanation of the above geometry. Specifically, the modality gap emerges during initialization as certain dimensions of image and text embeddings remain approximately constant in the embedding space, and the constants are distinct for images and texts due to separate initializations. During optimization, these constant dimensions lack a gradient that pushes them to align to the same value, leading to the preservation and orthogonality of the modality gap. Meanwhile, the alignment noise arises from the stable region produced by the contrastive loss, where points within a certain range result in a loss near zero, causing the optimization to halt.

Based on this understanding of the geometry, we propose a simple method, $C^3$, to improve the interchangeability of embeddings from different modalities, thereby enabling the learning of cross-modal tasks with uni-modal data. $C^3$ consists of three steps:

\vspace{-0.5em}
\begin{enumerate}[leftmargin=*]
\setlength{\itemsep}{0pt}
\setlength{\parskip}{0pt}
\setlength{\parsep}{0pt}
    \item \emph{Connect:} Related concepts from different modalities are connected via multi-modal contrastive learning, resulting in a shared representation space that can allow the interchangeability of embeddings from the different modalities. However, a modality gap and alignment noise exist.
    \item \emph{Collapse:} Based on its geometry, the modality gap can be closed by subtracting the embedding mean from each modality, eliminating the most dominant distributional difference between them.
    \item \emph{Corrupt:} Additional noise is introduced as regularization during training to improve the performance and robustness of learning cross-modal tasks from uni-modal data, given the alignment noise from multi-modal contrastive learning.
\end{enumerate}
\vspace{-0.5em}

We demonstrate the effectiveness and broad generalization of $C^3$ on four tasks: image, audio, video captioning and text-to-image generation, and achieve state-of-the-art performance on zero-shot evaluation settings when trained solely on uni-modal data. We also provide a detailed analysis of each component that contributes to performance improvements. Our method sheds new light on the possibility of achieving cross-modal tasks under a low-data regime, and our theoretical analysis provides insight into understanding multi-modal contrastive learning. 

Our contributions are three-fold:
\begin{enumerate}[leftmargin=*]
\setlength{\itemsep}{0pt}
\setlength{\parskip}{0pt}
\setlength{\parsep}{0pt}
\item We provide a theoretical explanation of the representation space geometry resulting from multi-modal contrastive learning.
\item Based on this geometry, we propose a simple three-step solution to enhance the interchangeability of embeddings from different modalities, improving cross-modal learning with uni-modal data.
\item We show the effectiveness of our method on image / audio / video captioning and text-to-image generation, achieving state-of-the-art results.
\end{enumerate}

\section{Learning Cross-Modal Tasks with Uni-Modal Data}
\label{sec:formulation}

In this section, we present the general notion of a cross-modal task and how we can leverage uni-modal data to learn such tasks.

\subsection{Cross-Modal Task Formulation}

Cross-modal tasks aim to learn a model that maps inputs from one modality $\gX$ (e.g., images) to another modality $\gY$ (e.g., texts). Given a paired multi-modal dataset $\gD = \{(x,y) \in \gX \times \gY\}$ (e.g., an image-caption dataset), the task can be achieved by minimizing the empirical risk $\gL_d$ between the predicted target $\hat{y}=g(f_\gX(x))$ and the true target $y$ over the dataset $\gD$, where $f_\gX : \gX \rightarrow \R^d$ is an encoder that maps inputs from $\gX$ to a $d$-dimensional representation space, and $g: \R^d \rightarrow \gY$ is a decoder that maps outputs from the encoder to $\gY$:
\begin{align*}
\min_{g,f_\gX} \frac{1}{|\gD|}\sum_{(x,y) \in \gD} \gL_d(\hat{y}, y), \quad \hat{y} = g(f_\gX(x))
\end{align*}
$\gL_d$ measures the discrepancy between the predicted and true targets, which can be mean squared error (MSE) for images and cross-entropy loss for texts.

\subsection{Enabling Cross-Modal Tasks with Uni-Modal Data}

The need for a multi-modal paired dataset $\gD$ to learn cross-modal tasks is suboptimal, as collecting such datasets can be expensive and time-consuming.
However, if we have a encoder $f_\gY : \gY \rightarrow \R^d$ that maps inputs from $\gY$ to the same representation space as the encoder $f_\gX$, i.e., $\forall x,y\in\gD, f_\gX(x) = f_\gY(y)$, we can train the cross-modal task using a uni-modal dataset $\gD' = \{y \in \gY\}$:
\begin{align*}
\min_{g} \frac{1}{|\gD'|}\sum_{y \in \gD'} \gL_d(\hat{y}, y), \quad \hat{y} = g(f_\gY(y))
\end{align*}
Note that $f_\gY$ should be frozen during training to maintain its embedding alignment with $f_\gX$. When evaluating in a cross-modal setting, we can replace $f_\gX$ with $f_\gY$, thus enabling cross-modal tasks with only uni-modal training data.

\subsection{Establishing a Shared Representation Space}

Recent advances in multi-modal contrastive learning has enabled for encoders that map similar inputs from different modalities to a shared representation space. Specifically, given a large multi-modal dataset\footnote{Acquiring domain-specific paired multi-modal datasets can be challenging. However, several works have gathered large-scale noisy image-caption pairs from the web and made pre-trained encoders $f_\gX$ and $f_\gY$ available for direct use.}, $n$ paired inputs are randomly sampled during each iteration and the following objective is optimized~\citep{radford2021learning}:
\begin{align*}
\min_{f_\gX, f_\gY}\gL &= -\frac{1}{2n} \sum_{i=1}^n \Big(\log \frac{\exp(\ve_{x_i} \cdot \ve_{y_i} / \tau)}{\sum_{j=1}^n \exp(\ve_{x_i} \cdot \ve_{y_j}  / \tau)}  + \log \frac{\exp(\ve_{x_i} \cdot \ve_{y_i} / \tau)}{\sum_{j=1}^n \exp(\ve_{x_j} \cdot \ve_{y_i}  / \tau)}\Big) 
\end{align*}
where $\ve_{x} = f_\gX(x)$, $\ve_{y}=f_\gY(y)$, and $\tau$ is the temperature hyperparameter. This loss function encourages high similarity between the embeddings of the paired image-texts relative to the similarities between unpaired ones. Intuitively, after optimizing the loss, paired image and text embeddings should collapse to the same point. However, empirical results have shown a significant gap between paired embeddings~\citep{liang2022mind}, which prevents the direct interchangeable use of image and text embeddings. In the next section, we provide a detailed analysis of the geometry of the multi-modal contrastive representation space.

\section{Multi-Modal Contrastive Representation Space Geometry}
\label{sec:theory}

We begin by providing the following proposition that describes the geometry of the multi-modal contrastive representation space. These findings inform our general method of how to adapt uni-modal data for cross-modal learning. 

\begin{proposition}
\textup{\textbf{(Multi-modal Contrastive Representation Space Geometry)}} \\
Given a paired image $x$ and text $y$, the relationship between the $\ell_2$-normalized image embedding $e_x$ and text embedding $e_y$ obtained from multi-modal contrastive learning can be described as:
\begin{align*}
\ve_x - \ve_y = \vc_\perp + \vepsilon
\end{align*}
where $\vc_\perp$ is a constant vector representing the \emph{modality gap} and is orthogonal to the image and text embedding span, i.e., $\forall \ve_{x_1}, \ve_{x_2}, \vc_\perp \cdot (\ve_{x_1} - \ve_{x_2}) = 0$; $\vepsilon \sim \gN(\bm{0}, \sigma^2 \bm{I})$ is a random Gaussian vector representing the \emph{alignment noise}.
\label{prop:geometry}
\end{proposition}

In the following subsections, we will first explain why the modality gap $\vc_\perp$ exists before and after model optimization. Then, we will introduce the alignment noise $\vepsilon$ and its relation to the temperature parameter in contrastive loss.

\subsection{Modality Gap}
\label{sec:gap}

The presence of a modality gap and its orthogonality to the image and text embedding span are due to the \emph{joint effect of initialization and optimization}. 

\paragraph{Initialization.} The modality gap exists when randomly initializing multi-modal models, which can be explained by the dimensional collapse~\citep{jing2022understanding} of the representation space defined below.
\begin{definition}
\textup{\textbf{(Dimensional Collapse of the Representation Space)}}\\
Given a $d$-dimensional representation space $\R^d$, we define its \emph{effective dimension} $d_e$ as:
\begin{align*}
d_e = \arg\min_{d'} \frac{\sum_{i=1}^{d'} \sigma_i}{\sum_{i=1}^d \sigma_i} \geq \gamma
\end{align*}
where the $\sigma_i$'s are the singular values of the representation covariance matrix in decreasing order, and $\gamma$ thresholds the minimum variance explained by the $d_e$ dimensions. \emph{Dimensional collapse} occurs when $d_e \ll d$.
\label{prop:collapse}
\end{definition}

\begin{wrapfigure}{r}{0.5\linewidth}
    \centering
    \includegraphics[width=\linewidth]{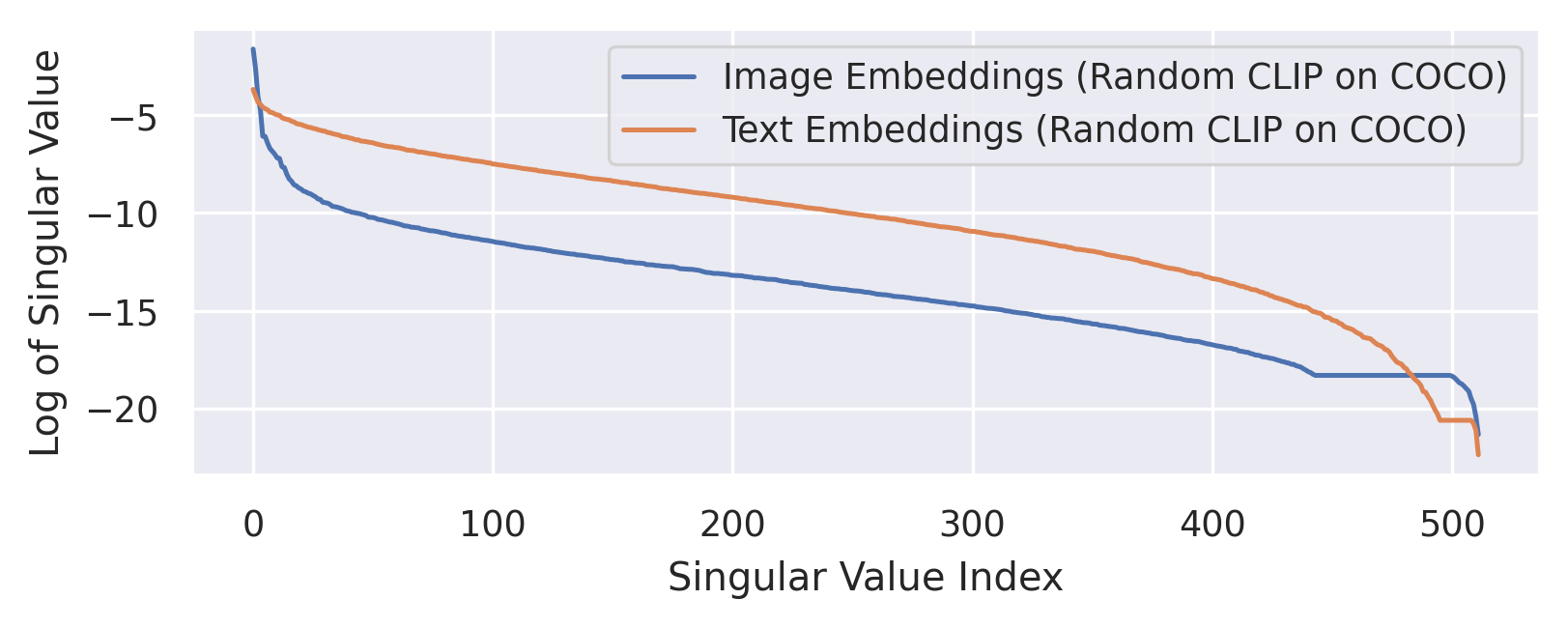}
    \vspace{-1.5em}
    \caption{\emph{Dimensional collapse of the CLIP representation space.} Singular values obtained from SVD reveal that the effective dimension of the image and text representation space is much smaller than the total number of dimensions.}
    \label{fig:svd}
    \vspace{-0.5em}
\end{wrapfigure}

To demonstrate the dimensional collapse phenomenon, we took a \emph{randomly initialized} CLIP with a $d=512$ representation space and fed MS-COCO images and captions as input. We obtained the corresponding image features and text features and performed SVD on the image feature and text feature covariance matrices. The distribution of the singular values is shown in Figure~\ref{fig:svd}, where we can clearly see that the effective dimension of both the image and text features is small. Specifically, when setting $\gamma=0.99$, the effective dimension of image embeddings is $d_{e,x}= 25$ and that of text embeddings is $d_{e,y} = 230$. Therefore, the effective dimension of the shared representation space $d_e \le d_{e,x} + d_{e,y} = 255$. The equality holds only when all the effective dimensions of the image and text are orthogonal.

Dimensional collapse indicates that only a small number of dimensions contribute significantly to the variance of the representation space, while the remaining dimensions can be viewed as constant. 
Suppose the shared representation space has a maximum effective dimension $d_e = 255$. This indicates that the image and text embeddings will remain constant in the $d_c = d - d_e = 257$ ineffective dimensions. As a result, a modality gap exists at the beginning of model optimization, as these $d_c$ ineffective dimensions will be inherently different for images and texts, given random initialization.

To verify this, we synthesize $n=1,000$ image and text embeddings in $d=512$ space. 
For each image embedding, we initialize the first $d_{e,x}$ dimensions, and for each text embedding, the $d_{e,x}$-th to $(d_{e,x} + d_{e,y})$-th dimensions, by randomly sampling a standard Gaussian distribution $\gN(0,1)$. All the other dimensions are set to a constant value drawn from a $\gN(0,1)$, where the constant for the image and text embeddings are different. Figure~\ref{fig:simulation} (left) illustrates this setup by showing the variance of each dimension. This setup mimics our findings of SVD analysis on CLIP, where we set image embeddings to have $d_{e,x}$ effective dimensions and text embeddings to have $d_{e,y}$ effective dimensions, and assume these effective dimensions are fully orthogonal. We normalize the embeddings to unit length before performing our analysis.

We observe a clear modality gap at the beginning: the $\ell_2$-distance between the mean of the image embeddings and the mean of the text embeddings is $1.21$. When we only consider the last $d_c$ ineffective dimensions, the average distance is $0.99$. More discussion is in Appendix~\ref{sec:appendix_collapse}.

\paragraph{Optimization.} The analysis above reveals that the modality gap exists at initialization, and we further analyze why optimizing for the multi-modal contrastive loss fails to close the gap. The following lemma reveals that there is no gradient in the modality gap direction, therefore the gap and its orthogonality will be preserved.

\begin{lemma}
\textup{\textbf{(Gradients in Contrastive Optimization)}} (Proof in Appendix~\ref{sec:appendix_proof}) \\
With the mild assumption of equal presence of $n$ images and texts with $p(x_i)=p(y_i)=1/n$, optimizing the multi-modal contrastive loss $\gL = -\frac{1}{2n} \sum_{i=1}^n \big(\log \frac{\exp(\ve_{x_i} \cdot \ve_{y_i} / \tau)}{\sum_{j=1}^n \exp(\ve_{x_i} \cdot \ve_{y_j}  / \tau)} + \log \frac{\exp(\ve_{x_i} \cdot \ve_{y_i} / \tau)}{\sum_{j=1}^n \exp(\ve_{x_j} \cdot \ve_{y_i}  / \tau)}\big)$ yields the following gradients:
\begin{align*}
\nabla_{\ve_{x_i}} \gL &= \lambda \sum_{j=1}^n \alpha_{y_j} (\ve_{y_j} - \ve_{y_i}), \quad\quad\nabla_{\ve_{y_i}} \gL = \lambda \sum_{j=1}^n \alpha_{x_j} (\ve_{x_j} - \ve_{x_i})
\end{align*}
\noindent where $\lambda=1 / (2n\tau)$, $\alpha_{x_j} = p(x_j|y_i)+p(y_i|x_j)$, 
$\alpha_{y_j} = p(y_j|x_i)+p(x_i|y_j)$, 
$p(x_i|y_j)=\frac{\exp(\ve_{x_i} \cdot \ve_{y_j} / \tau)} {\sum_{k=1}^n \exp(\ve_{x_k} \cdot \ve_{y_j} / \tau)}$, 
$p(y_i|x_j) = \frac{\exp(\ve_{y_i} \cdot \ve_{x_j} / \tau)}{\sum_{k=1}^n \exp(\ve_{y_k} \cdot \ve_{x_j} / \tau)}$, and $\tau$ is temperature.
\label{prop:gradient}
\end{lemma}

Lemma~\ref{prop:gradient} highlights that the gradients of image embeddings during contrastive optimization are fully determined by the text embedding span, and vice versa. Due to the dimensional collapse of the image and text embedding span, the contrastive optimization process fails to propagate gradients in the direction of the ineffective dimensions, resulting in gap preservation and orthogonality to the image and text embedding span after optimization. Lemma~\ref{prop:gradient} also implies that the effective dimensionality of the joint representation space remains unchanged after optimization. 

To empirically verify this, we optimize the multi-modal contrastive loss $\gL$ on the $n=1,000$ previously synthesized image and text embeddings. We optimize for $200$K steps with a learning rate $0.1$, and CLIP's initial temperature of $\tau=0.07$. We compute the variance of each dimension pre- and post-contrastive optimization and plot the results in Figure~\ref{fig:simulation} (right). From the figure, we can see that the first $d_e = 255$ effective dimensions are aligned after optimization while the last $d_c=257$ ineffective ones remain unchanged. 
This verifies that no gradient is propagated to the ineffective dimensions, as the variance for these dimensions remains zero after optimization. The modality gap becomes slightly smaller (0.82 compared to 0.99 before optimization) mainly due to $\ell_2$-regularization, where changes in the effective dimensions affect changes in the ineffective ones.

\begin{figure}[!tb]
    \centering
    \includegraphics[width=0.49\linewidth]{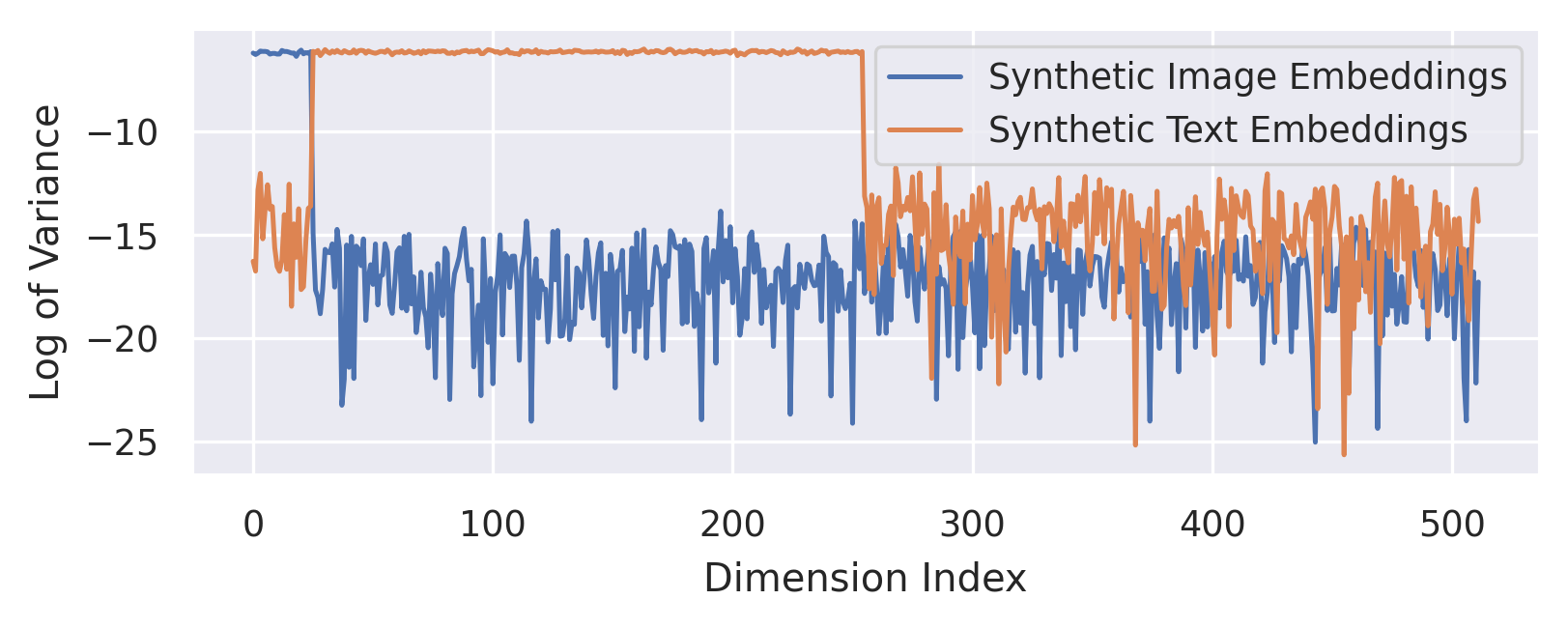}
    \includegraphics[width=0.49\linewidth]{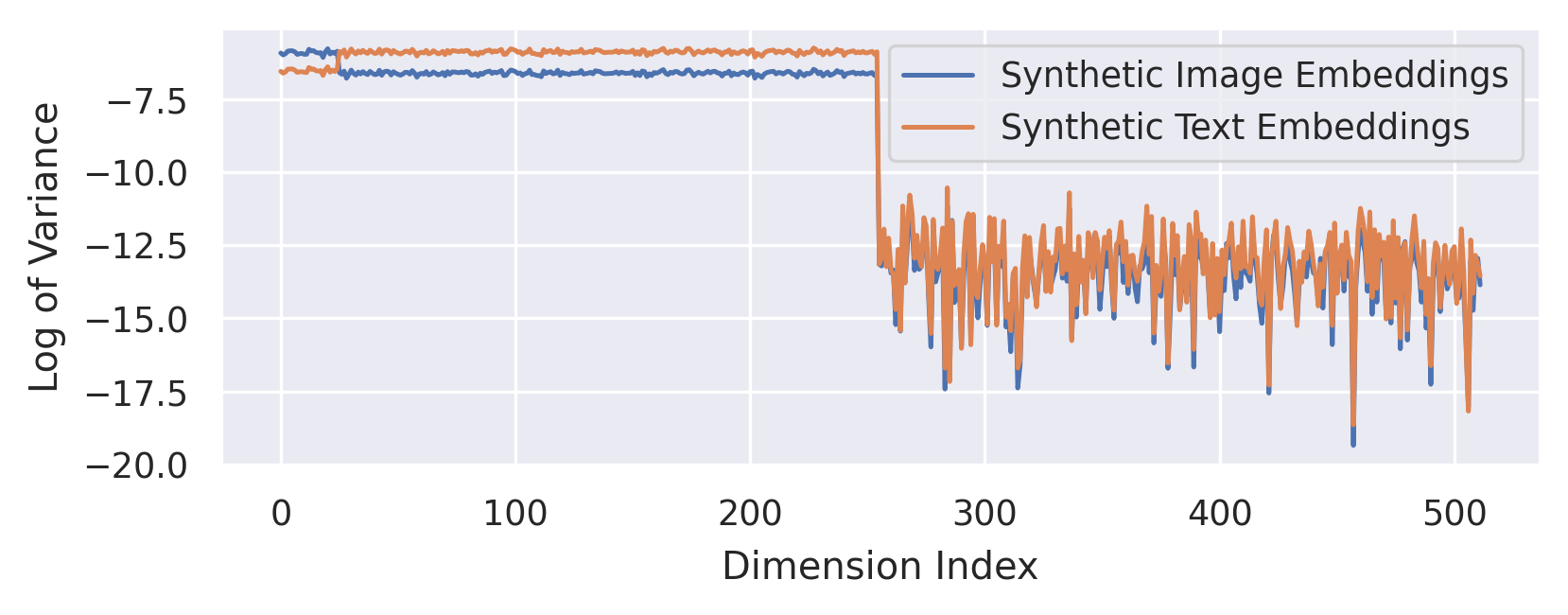}
    \vspace{-1.5em}
    \caption{\emph{Variance of each dimension before (left) and after (right) multi-modal contrastive optimization.} Our analysis reveals that gradients will only be propagated to effective dimensions and no gradient will be propagated to ineffective dimensions. Therefore, the effective dimensions are aligned while ineffective dimensions remain constant after optimization.}
    \label{fig:simulation}
    \vspace{-0.5em}
\end{figure}

In summary, due to dimensional collapse at model initialization, there are ineffective dimensions where image and text embeddings can be viewed as different constants, resulting in a modality gap at initialization. During optimization, these ineffective dimensions have no gradient update, and thus, the modality gap and its orthogonality are preserved after optimization.

\subsection{Alignment Noise}
\label{sec:misalignment}

In this section, we explain alignment noise after multi-modal contrastive learning. This noise results from the stable region of contrastive loss demonstrated in the following lemma, where we can consider a single term in the loss given the symmetry of the added terms.

\begin{lemma}
\textup{\textbf{(Stable Region Controlled by Temperature)}} (Proof in Appendix~\ref{sec:appendix_proof}) \\
We consider a single term in the multi-modal contrastive loss
$\gL_i=-\log \frac{\exp(\ve_{x_i} \cdot \ve_{y_i} / \tau)}{\sum_{j=1}^n \exp(\ve_{x_i} \cdot \ve_{y_j} / \tau)}$. We define the \emph{margin} $r = \ve_{x_i} \cdot \ve_{y_i} - \max_{j \ne i} \ve_{x_i} \cdot \ve_{y_j}$ as the measure of the similarity difference between the matched pair and the hardest negative pair. When $r$ exceeds a threshold given below, $\gL_i$ falls below a small pre-set value $\delta$, where we assume optimization ends:
\begin{align*}
r \ge \tau \log \frac{o(\tau)}{\exp (\delta) - 1},
\end{align*}
where $o(\tau)$ is a monotonically increasing function of temperature $\tau$ that satisfies $1 < o(\tau) < n$. Therefore, the required margin $r$ is monotonically increasing with $\tau$.
\label{prop:region}
\end{lemma}

\begin{wrapfigure}{r}{0.5\linewidth}
    \centering
    \includegraphics[width=\linewidth]{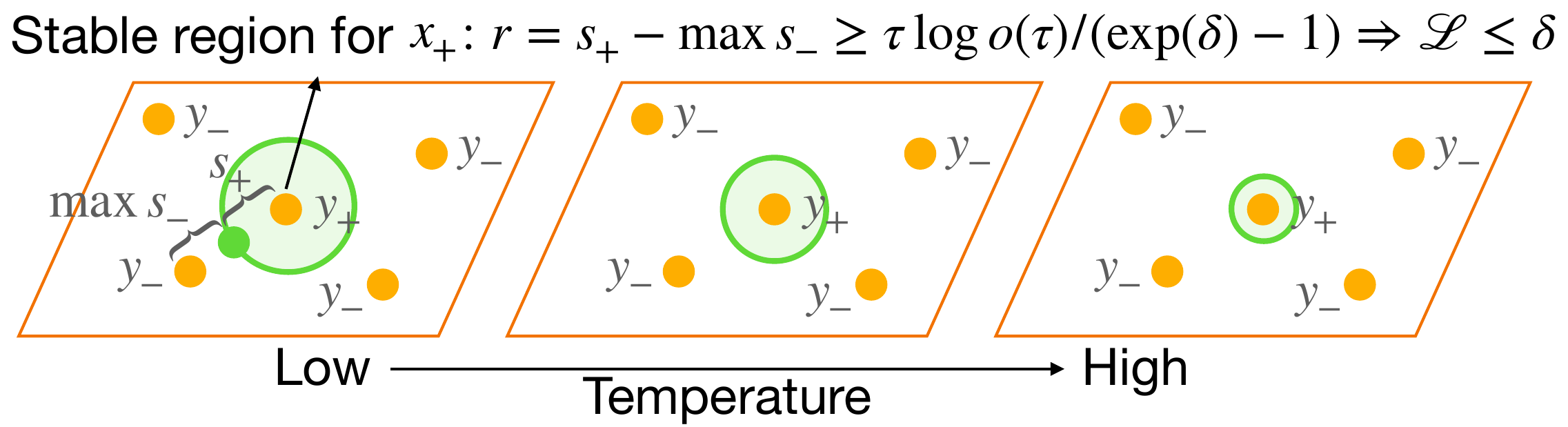}
    \vspace{-1.5em}
    \caption{\emph{Stable region (green area) of contrastive learning controlled by temperature.} Within the stable region, the loss falls below a small preset value, indicating that optimization has ended. The region increases as the temperature decreases. }
    \label{fig:region}
    \vspace{-0.5em}
\end{wrapfigure}

Lemma~\ref{prop:region} suggests that there is a stable region of the contrastive loss. This region can be viewed as a function of temperature where region size increases as the temperature decreases, with the required margin becoming smaller. Within this stable region, the loss falls below the small value $\delta$, indicating that optimization has ended. In the extreme case where $\tau \rightarrow 0_+$, if $r \ge 0$, the loss will be less than $\delta$ ($\delta \rightarrow 0_+$ for this case). This means that given all the $y_j$, $x_i$ can end up within a region instead of a fixed point, resulting in the same zero loss. Figure~\ref{fig:region} illustrates the stable region defined by the margin with regard to the temperature. Therefore, there may be a mismatch between the matched pairs in the representation space, resulting in alignment noise. 

\subsection{Summary}

\begin{wraptable}{r}{0.23\linewidth}
    \vspace{-3em}
    \centering
    \small
    \setlength\tabcolsep{2pt}
    \resizebox{0.23\columnwidth}{!}{
    \begin{tabular}{l|cc}
    \toprule
    \textbf{Statistic} & \textbf{Mean} & \textbf{Std} \\
    \midrule
    $\|\vd^{(i)}\|_2$ & 0.83 & 0.01 \\
    $\cos(\vd^{(i)}, \vd^{(j)}) $ & 0.99 & 0.00 \\
    $\cos(\vd^{(i)}, \vr_{j,k}^{(i)})]$ & 0.00 & 0.06 \\
    $\E [\vepsilon_j^{(i)}]_k$ & 0.00 & 0.00 \\
    $\cos(\vepsilon_j^{(i)}, \vepsilon_k^{(i)}) $ & 0.00 & 0.10 \\
    \bottomrule
    \end{tabular}
    }
    \vspace{-1em}
    \caption{Statistics that reveals representation space geometry.}
    \label{tab:statistics}
    \vspace{-3em}
\end{wraptable}

In Section~\ref{sec:gap} and~\ref{sec:misalignment}, we explain the modality gap $\vc_\perp$ and alignment noise $\vepsilon$ in Proposition~\ref{prop:geometry}, respectively. Combining them together, we explain the geometric relation of paired embeddings $\ve_x - \ve_y = \vc_\perp + \vepsilon$. 

We verify this geometric relation using CLIP on the MS-COCO image-caption dataset. We randomly group each 100 images into group $i$, and define individual gap $\vd_j^{(i)} = \ve_{x_j}^{(i)} - \ve_{y_j}^{(i)}$, group gap $\vd^{(i)} = \E_j [\vd_j^{(i)}]$, image difference $\vr_{j,k}^{(i)} = \ve_{x_j}^{(i)} - \ve_{x_k}^{(i)}$, alignment noise $\vepsilon_{j}^{(i)} = \vd_{j}^{(i)} - \vd^{(i)}$, where $j,k$ are image or text indices. These statistics are computed in Table~\ref{tab:statistics}, where the first three statistics show that the modality gap $\vc_\perp$ approximates a constant vector orthogonal to the image and text embedding span, and the last two statistics show that the alignment noise $\vepsilon$ can be viewed as Gaussian noise. We provide a detailed explanation of how to interpret these statistics in Appendix~\ref{sec:empirical_verification}.

This geometric analysis in this section serves as the foundation of the approach we introduce in the next section, where we develop a simple method to align the shared representation space and enable learning cross-modal tasks with uni-modal data.
\section{Connect, Collapse, Corrupt}
\label{sec:method}

Proposition~\ref{prop:geometry} from Section~\ref{sec:theory} reveals the geometry of the multi-modal contrastive representation space. Based on this, we propose three steps, \emph{connect, collapse, corrupt ($C^3$)}, to align the representation space, making it possible for embeddings from different modalities to be interchangeably consumed by the decoder and thus enabling learning cross-modal tasks from uni-modal data.

\paragraph{Stage 1: Connect.}

This stage establishes connections between similar concepts across different modalities. We leverage recent advances in multi-modal contrastive learning~\citep{radford2021learning} and use encoders trained with this strategy to build cross-modal models. However, a modality gap and alignment noise exists after multi-modal contrastive learning, as shown in Proposition~\ref{prop:geometry}.

\paragraph{Stage 2: Collapse.}

When directly using embeddings of different modalities as input, there is a drastic degradation in performance due to the modality gap, which causes input distributions to the decoder to differ. To address this issue, we adopt a simple approach proposed by~\citep{zhang2023diagnosing} that effectively removes the modality gap. Specifically, during training, in place of $\ve_x$, we feed in $\ve_x' = \ve_x - \E_x [\ve_x]$ to the decoder, and during inference with another modality, in place of $\ve_y$, we feed in $\ve_y' = \ve_y - \E_y [\ve_y]$. This approach collapses the modality gap, eliminating the input distribution mismatch between the two modalities:
\begin{align*}
\ve_x' - \ve_y' = (\ve_x - \ve_y) - (\E_x [\ve_x] - \E_y [\ve_y]) = \vepsilon
\end{align*}

\paragraph{Stage 3: Corrupt.}

After removing the modality gap, there is still alignment noise which can be approximated by a $\gN(0, \sigma^2 I)$. During unsupervised training, instead of directly decoding $y$ from $\ve_y'$, we add explicit Gaussian noise to the input and decode $y$ from $\ve_y'' = \ve_y' + \vepsilon$ following~\citep{nukrai-etal-2022-text,zhou2022lafite}. By introducing this noise, the uni-modal and multi-modal training processes become similar, and the learned decoder is more robust and invariant to the small perturbation $\gN(0, \sigma^2 I)$, leading to improved performance.

Appendix Algorithm~\ref{alg:c3} summarizes the entire procedure of our proposed method, $C^3$, that enables learning cross-modal tasks with uni-modal data.

\section{Results}
\label{sec:results}

In this section, we verify the effectiveness of our proposed method, $C^3$, on four tasks: image captioning, audio captioning, video captioning, and text-to-image generation. We show that our method achieves state-of-the-art performances, generalizes to different modalities and contrastive embedding spaces, and is especially useful when multi-modal data are limited.

\subsection{Image Captioning}

We use the ClipCap model~\citep{mokady2021clipcap}, pairing a frozen CLIP ViT-B/32 image encoder~\citep{radford2021learning} with a GPT-2 decoder~\citep{radford2019language}. A lightweight MLP mapping network bridges the dimensional gap between CLIP (512-$d$) and GPT-2 (768-$d$) embeddings and also produces a prefix for GPT-2 caption generation. We train and evaluate on the MS-COCO dataset~\citep{mscoco} using the standard split~\citep{karpathysplit}, comprising 113K training images and 5K each for validation and testing, with each image having 5 captions. We utilize metrics such as BLEU~\citep{papineni2002bleu} and ROUGE~\citep{lin2004rouge} to evaluate lexical and semantic similarity between generated and human captions. 

We first train our model for text reconstruction using the MS-COCO captions only. Following $C^3$, we extract the text embedding from the frozen CLIP text encoder and apply the collapse operation (remove pre-computed mean) and corrupt operation (add Gaussian noise). After training, we evaluate our model in the cross-modal setting, replacing the text encoder with the CLIP image encoder and decoding captions from image embeddings. We refer to this evaluation setting as \emph{image-free zero-shot} evaluation, as images are not seen during training. Additionally, we fine-tune the pre-trained model on different amounts of image-caption pairs and evaluate its performance. We refer to this evaluation setting as \emph{semi-supervised} evaluation. More details can be found in Appendix~\ref{sec:appendix_i2t}. We show image-free zero-shot captioning results in Table~\ref{tab:i2t} and semi-supervised captioning results in Figure~\ref{fig:active}. 

\paragraph{$C^3$ achieves state-of-the-art image-free zero-shot captioning results.}
As shown in Table~\ref{tab:i2t}, our proposed method, $C^3$, outperforms previous state-of-the-art methods in image-free captioning. Details of the other methods can be found in Appendix Section~\ref{sec:related_works}. Our ablation analysis demonstrates that both the collapse and corrupt components are crucial for improving cross-modal evaluation performance, as they eliminate the differences between embeddings from different modalities. Notably, the most competitive baseline, CapDec~\citep{nukrai-etal-2022-text}, can be viewed as an ablated version of $C^3$, but without an analysis of why the corruption works. Our proposed method, however, provides a clear explanation based on a geometric analysis of the multi-modal embedding space, and we further improve performance by introducing a collapse step. Overall, $C^3$ represents a potential standard approach for future works that use a multi-modal contrastive embedding space.

\begin{table*}[htbp]
\centering
\small
\setlength\tabcolsep{2pt}
\resizebox{\columnwidth}{!}{
\begin{tabular}{l|ccc|cccccc}
\toprule

\textbf{Method} & \textbf{Conn.} & \textbf{Coll.} & \textbf{Corr.} & \textbf{BLEU-1}$_\uparrow$  & \textbf{BLEU-4}$_\uparrow$  & \textbf{METEOR}$_\uparrow$  & \textbf{ROUGE-L}$_\uparrow$ & \textbf{CIDEr}$_\uparrow$  & \textbf{SPICE}$_\uparrow$  \\
\midrule
\multicolumn{10}{c}{\emph{Baselines}} \\
ZeroCap~(\citeyear{tewel2022zerocap}) & \ding{55} & \ding{55} & \ding{55} & 49.8 & 7.0 & 15.4 & 31.8 & 34.5 & - \\
MAGIC~(\citeyear{su2022language}) & \ding{55} & \ding{55} & \ding{55} & 56.8 & 12.9 & 17.4 & 39.9 & 49.3 & 11.3 \\
ESPER~(\citeyear{yu2022multimodal}) & \ding{55} & \ding{55} & \ding{55} & - & 21.9 & 21.9 & - & 78.2 & - \\
CLIPRe~(\citeyear{li2023decap}) & \ding{51} & \ding{55} & \ding{55} & - & 4.6 & 13.3 & - & 25.6 & 9.2 \\
DeCap~(\citeyear{li2023decap}) & \ding{51} & \ding{55} & \ding{55} & - & 8.9 & 17.5 & - & 50.6 & 13.1 \\
WS-ClipCap~(\citeyear{tam2023simple}) & \ding{51} & \ding{55} & \ding{55} & 50.3 & 9.6 & 15.2 & 37.5 & 33.7 & 8.6 \\
WS-ClipCap-Multi~(\citeyear{tam2023simple}) & \ding{51} & \ding{55} & \ding{51} & 65.5 & 22.1 & 22.2 & 48.0 & 74.6 & 14.9 \\
CapDec~(\citeyear{nukrai-etal-2022-text}) & \ding{51} & \ding{55} & \ding{51} & 69.2 & 26.4 & \textbf{25.1} & 51.8 & 91.8 & - \\
\midrule
\multicolumn{10}{c}{\emph{Ours}} \\
$C^1$ & \ding{51} & \ding{55} & \ding{55} & 28.1 & 2.4 & 12.2 & 25.4 & 13.0 & 6.8 \\
$C^2_1$ & \ding{51} & \ding{51} & \ding{55} & 44.4 & 6.1 & 15.5 & 33.6 & 25.2 & 9.2 \\ 
$C^2_2$ & \ding{51} & \ding{55} & \ding{51} & 69.0 & 25.5 & 24.3 & 50.8 & 87.6 & 17.6 \\
$C^3$ & \ding{51} & \ding{51} & \ding{51} & \textbf{71.0}\tiny{$\pm$0.1} & \textbf{27.7}\tiny{$\pm$0.1} & 25.0\tiny{$\pm$0.0} & \textbf{52.0}\tiny{$\pm$0.0} & \textbf{93.3}\tiny{$\pm$0.3} & \textbf{18.3}\tiny{$\pm$0.1} \\

\bottomrule
\end{tabular}
}
\vspace{-1em}
\caption{\emph{Image-free image-to-text captioning results.} We achieve state-of-the-art zero-shot image captioning and our ablation shows the effectiveness of each component in our method. Baseline results are from the original paper, where some metrics were not reported. Std with 3 runs reported.} 
\label{tab:i2t}
\end{table*}

\begin{figure*}[htbp]
    \centering
    \includegraphics[width=\linewidth]{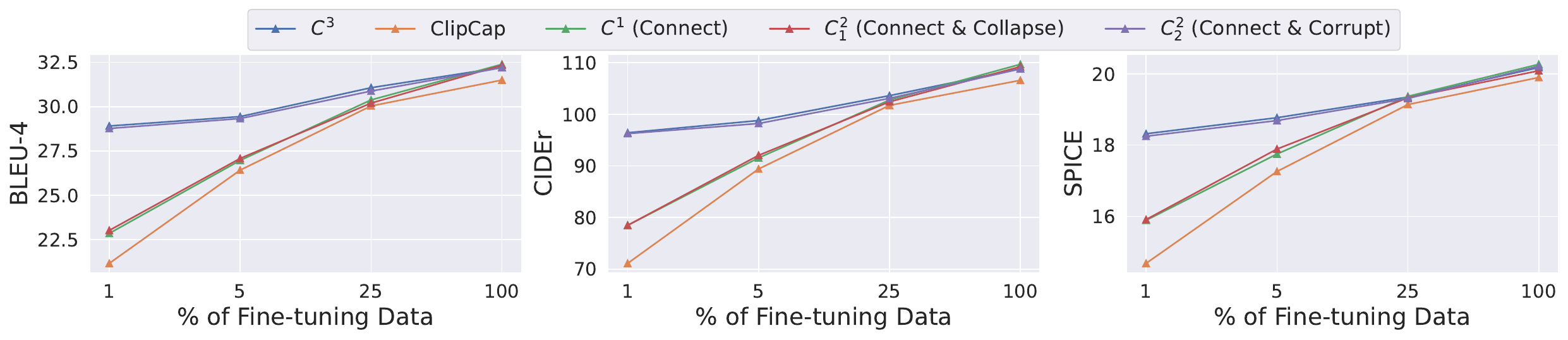}
    \vspace{-2em}
    \caption{\emph{Image-to-text captioning results in the low data regime.} When paired multi-modal data are limited, our approach that leverages uni-modal data for pre-training leads to substantial improvements compared to the purely supervised method (ClipCap).}
    \label{fig:active}
    \vspace{-1em}
\end{figure*}

\paragraph{$C^3$ is particularly useful in low-data regimes.} In scenarios where multi-modal data is limited, our method remains highly effective. To demonstrate this, we fine-tuned our pre-trained model on 1\%, 5\%, 25\%, and 100\% of the MS-COCO training image-text pairs, and compared our performance with a fully supervised baseline (ClipCap) and ablated models (see Figure~\ref{fig:active}). The results clearly show that $C^3$ outperforms the fully supervised baseline across all metrics, with the most significant improvements seen in low-data regimes where multi-modal paired data are limited. Thus, our method represents a promising solution for achieving cross-modal tasks in such scenarios.

\paragraph{Qualitative analysis of collapse and corrupt.} Both the collapse and corrupt components of our method show consistent improvements, but it is not immediately clear how they do so. To address this, we provide qualitative results in Appendix~\ref{sec:appendix_i2t}. We can see that after collapsing, the generated captions are much more natural and fluent, as it removes the most significant distributional difference between image and text embeddings. After corrupting, the model generates more accurate and faithful text descriptions of the image. We hypothesize that adding noise makes the decoder robust to small variations in the embedding space. Therefore when evaluating in cross-modal settings, the alignment noise will not affect the prediction and thus reduces hallucination in the generated caption.

\subsection{Text-to-Image Generation}

\begin{wraptable}{r}{0.46\linewidth}
\vspace{-1em}
\centering
\small
\setlength\tabcolsep{2pt}
\resizebox{0.46\columnwidth}{!}{
\begin{tabular}{l|ccc|cc}
\toprule

\textbf{Method} & \textbf{Conn.} & \textbf{Coll.} & \textbf{Corr.} & \textbf{FID}$_\downarrow$ & \textbf{IS}$_\uparrow$\\
\midrule

\multicolumn{6}{c}{\emph{Baselines}} \\
DALL-E~(\citeyear{ramesh2021zero}) & \ding{55} & \ding{55} & \ding{55} & 27.5 & 17.9 \\
CogView~(\citeyear{ding2021cogview}) & \ding{55} & \ding{55} & \ding{55} & 27.1 & 18.2 \\
LAFITE$_G$~(\citeyear{zhou2022lafite}) & \ding{51} & \ding{55} & \ding{51} & 20.9 & 24.9 \\
\midrule
\multicolumn{6}{c}{\emph{Ours}} \\
$C^1$ & \ding{51} & \ding{55} & \ding{55} & 29.8 & 22.4 \\
$C^2_1$ & \ding{51} & \ding{51} & \ding{55} & 21.7 & 24.4 \\ 
$C^2_2$ & \ding{51} & \ding{55} & \ding{51} & 19.8 & 25.5 \\
$C^3$ & \ding{51} & \ding{51} & \ding{51} & \textbf{19.6} & \textbf{26.0} \\
\bottomrule
\end{tabular}
}
\vspace{-1em}
\caption{\emph{Language-free text-to-image generation results.} Our method $C^3$ consistently outperforms the baselines.}
\label{tab:t2i}
\end{wraptable}

We further apply our method, $C^3$, to text-to-image generation, the reverse of image captioning. We utilize LAFITE~\citep{zhou2022lafite}, which integrates a frozen CLIP ViT-B/32~\citep{radford2021learning} as the text encoder and a modified StyleGAN2's generator~\citep{Karras2019stylegan2} as the trainable decoder. We use the MS-COCO dataset~\citep{mscoco} with LAFITE's official split, including 82K training and 40K validation images, with 5 captions each. We use the standard metrics such as FID (Fréchet Inception Distance)~\citep{heusel2017gans} and IS (Inception Score)~\citep{salimans2016improved} to assess the realism of generated images. Similar to the image captioning setting, we first train the model for image reconstruction using images only, and then evaluate the model in the cross-modal setting by generating images from text. We show the \emph{language-free zero-shot} image generation results in Table~\ref{tab:t2i}. Similar to image captioning, we find that our method \textbf{$C^3$ consistently outperforms the baselines} in terms of FID and IS. Our ablation study further reveals that each component of $C^3$ is useful for improving the performance. More detailed setups and qualitative comparisons can be found in Appendix~\ref{sec:appendix_t2i}.

\subsection{Generalization to Other Modalities and Embedding Spaces}

To verify the generalization of our method to other modalities, datasets, and embedding spaces, we further conduct experiments on zero-shot captioning from \emph{image}, \emph{audio} and \emph{video} using \emph{ImageBind}~\citep{girdhar2023imagebind} embeddings. We use the same settings as image captioning with CLIP. Results are shown in Table~\ref{tab:result_generalization}. We find that \textbf{$C^3$ consistently improves baselines in all the settings}, and \textbf{using ImageBind embeddings achieves further improvements in image captioning compared to CLIP embeddings}.

\begin{table}[!tb]
\small
\centering
\setlength\tabcolsep{2pt}
\resizebox{\columnwidth}{!}{
\begin{tabular}{c|ccc|ccc|ccc}
\toprule
\multirow{2}{*}{} & \multicolumn{3}{c|}{\textbf{Image Captioning (MS-COCO~\citeyear{mscoco})}} & \multicolumn{3}{c|}{\textbf{Audio Captioning (Clotho~\citeyear{clotho})}} & \multicolumn{3}{c}{\textbf{Video Captioning (MSR-VTT~\citeyear{xu2016msr-vtt})}} \\
 & \textbf{BLEU-1$_\uparrow$} &\textbf{METEOR$_\uparrow$} & \textbf{ROUGE-L$_\uparrow$} & \textbf{BLEU-1$_\uparrow$} &\textbf{METEOR$_\uparrow$} & \textbf{ROUGE-L$_\uparrow$} & \textbf{BLEU-1$_\uparrow$} &\textbf{METEOR$_\uparrow$} & \textbf{ROUGE-L$_\uparrow$} \\
\midrule
$C^1$ & 33.5 & 12.9 & 25.8 & 21.8 & 17.3 & 18.1 & 16.7 & 12.6 & 15.2 \\
$C^2_1$ & 53.8 & 17.4 & 38.6 & 26.7 & 20.0 & 21.4 & 25.1 & 17.8 & 23.1 \\
$C^2_2$ & 64.4 & 22.2 & 45.8 & 27.6 & 20.0 & 20.6 & 25.2 & 18.1 & 23.8 \\
$C^3$ & \textbf{74.0} & \textbf{26.6} & \textbf{54.0} & \textbf{29.5} & \textbf{20.1} & \textbf{23.0}  & \textbf{31.4} & \textbf{20.0} & \textbf{26.9} \\
\bottomrule
\end{tabular}
}
\vspace{-1em}
\caption{\emph{Generalization of $C^3$ to other modalities, datasets, and contrastive embedding spaces.}}
\vspace{-3.3mm}
\label{tab:result_generalization}
\end{table}

\section{Related Works \small{(Full Version in Appendix~\ref{sec:related_works})}}

\paragraph{Multi-modal contrastive learning and resulting geometry.} 

Multi-modal contrastive learning aims to bridge representations from different modalities, drawing similar concepts closer and distancing dissimilar ones~\citep{radford2021learning}. CLIP~\citep{radford2021learning}, ImageBind~\citep{girdhar2023imagebind}, and similar models have leveraged extensive multi-modal data to construct such representation spaces, which have been demonstrated to effectively support a range of uni-modal and multi-modal applications~\citep{wortsman2022robust,ramesh2022hierarchical,shen2022how}. However, the resulting geometry in the shared representation space, particularly the ``modality gap'', where embeddings from different modalities are clearly separate in the shared representation space, remains under-explored~\citep{liang2022mind,zhang2023diagnosing}. In our work, we unify the observations from~\citet{liang2022mind} and~\citet{zhang2023diagnosing}, and contribute a formal formulation and theoretical explanation of the unique geometry resulting from multi-modal contrastive learning. 

\paragraph{Learning cross-modal tasks with uni-modal data.}

Given the expense of multi-modal data collection, there is a growing interest in learning cross-modal tasks using uni-modal data. Based on the assumption that contrastive optimization makes representations from different modalities interchangeable, recent works have leveraged these representation spaces and shown great success in building image captioning models with text data only~\citep{tam2023simple,li2023decap,nukrai-etal-2022-text} and text-to-image generation models with image data only~\citep{zhou2022lafite,zhou2022lafite2,zhou2022shifted}. 
Despite these advancements, these methods have noted the intriguing ``modality gap'' phenomenon and proposed different empirical methods to address this gap, such as ~\citep{tam2023simple}'s paraphrased decoding,~\citep{li2023decap,zhou2022lafite2}'s memory retrieval,~\citep{nukrai-etal-2022-text,zhou2022lafite}'s noise addition, and~\citep{zhou2022shifted}'s prior network. In our work, we first provide a theoretical analysis of the multi-modal representation space geometry. Based on the geometry, we propose a simple method that addresses the ``modality gap'' in a principled manner and ultimately improves performance on cross-modal tasks and outperforms these strategies.

\section{Conclusion}
\label{sec:conclusion}

In this work, we provide a theoretical explanation of the unique geometry that arises from multi-modal contrastive learning. Building upon this, we present a straightforward technique, $C^3$, which enhances the interchangeability of embeddings between modalities, enabling the creation of cross-modal applications using only uni-modal data. We demonstrate the effectiveness of our approach on image, audio, video captioning and text-to-image generation, achieving state-of-the-art performance on zero-shot evaluation settings when trained solely on uni-modal data.

\section*{Acknowledgments}
We thank all the reviewers for their constructive feedback.
Serena Yeung-Levy is a Chan Zuckerberg Biohub – San Francisco Investigator.

\section*{Ethics Statement}
Drawing from a deep understanding of the multi-modal contrastive representation space, our method enables the effortless creation of multi-modal content from uni-modal data. While this has the potential to revolutionize content generation, the risks associated with generation, such as creating harmful content, persists. We strongly emphasize the importance of ethically and responsibly employing any advancements built upon our approach, ensuring that their impact enhances the betterment of our digital ecosystem.

\section*{Reproducibility Statement}
We provide open-source implementation of our work at \url{https://github.com/yuhui-zh15/C3}. The implementations will enable researchers to reproduce all the experiments described here and run their own analyses on additional multi-modal models and datasets.

\bibliography{iclr2024_conference}

\begin{thebibliography}{43}
\providecommand{\natexlab}[1]{#1}
\providecommand{\url}[1]{\texttt{#1}}
\expandafter\ifx\csname urlstyle\endcsname\relax
  \providecommand{\doi}[1]{doi: #1}\else
  \providecommand{\doi}{doi: \begingroup \urlstyle{rm}\Url}\fi

\bibitem[Anderson et~al.(2016)Anderson, Fernando, Johnson, and
  Gould]{anderson2016spice}
Peter Anderson, Basura Fernando, Mark Johnson, and Stephen Gould.
\newblock Spice: Semantic propositional image caption evaluation.
\newblock In \emph{European Conference on Computer Vision (ECCV)}, 2016.

\bibitem[Banerjee \& Lavie(2005)Banerjee and Lavie]{banerjee2005meteor}
Satanjeev Banerjee and Alon Lavie.
\newblock Meteor: An automatic metric for mt evaluation with improved
  correlation with human judgments.
\newblock In \emph{Workshop of Annual Meeting of the Association for
  Computational Linguistics (ACL Workshop)}, 2005.

\bibitem[Ding et~al.(2021)Ding, Yang, Hong, Zheng, Zhou, Yin, Lin, Zou, Shao,
  Yang, et~al.]{ding2021cogview}
Ming Ding, Zhuoyi Yang, Wenyi Hong, Wendi Zheng, Chang Zhou, Da~Yin, Junyang
  Lin, Xu~Zou, Zhou Shao, Hongxia Yang, et~al.
\newblock Cogview: Mastering text-to-image generation via transformers.
\newblock \emph{Conference on Neural Information Processing Systems (NeurIPS)},
  2021.

\bibitem[Drossos et~al.(2020)Drossos, Lipping, and Virtanen]{clotho}
Konstantinos Drossos, Samuel Lipping, and Tuomas Virtanen.
\newblock Clotho: An audio captioning dataset.
\newblock In \emph{International Conference on Acoustics, Speech and Signal
  Processing (ICASSP)}, 2020.

\bibitem[EleutherAI(2021)]{CLASP}
EleutherAI.
\newblock Clasp: Contrastive language aminoacid sequence pretraining, 2021.
\newblock URL \url{https://github.com/MicPie/clasp}.

\bibitem[Eyuboglu et~al.(2022)Eyuboglu, Varma, Saab, Delbrouck, Lee-Messer,
  Dunnmon, Zou, and Re]{eyuboglu2021domino}
Sabri Eyuboglu, Maya Varma, Khaled~Kamal Saab, Jean-Benoit Delbrouck,
  Christopher Lee-Messer, Jared Dunnmon, James Zou, and Christopher Re.
\newblock Domino: Discovering systematic errors with cross-modal embeddings.
\newblock In \emph{International Conference on Learning Representations
  (ICLR)}, 2022.

\bibitem[Girdhar et~al.(2023)Girdhar, El-Nouby, Liu, Singh, Alwala, Joulin, and
  Misra]{girdhar2023imagebind}
Rohit Girdhar, Alaaeldin El-Nouby, Zhuang Liu, Mannat Singh, Kalyan~Vasudev
  Alwala, Armand Joulin, and Ishan Misra.
\newblock Imagebind: One embedding space to bind them all.
\newblock In \emph{Conference on Computer Vision and Pattern Recognition
  (CVPR)}, 2023.

\bibitem[Hernandez et~al.(2022)Hernandez, Schwettmann, Bau, Bagashvili,
  Torralba, and Andreas]{hernandez2022natural}
Evan Hernandez, Sarah Schwettmann, David Bau, Teona Bagashvili, Antonio
  Torralba, and Jacob Andreas.
\newblock Natural language descriptions of deep visual features.
\newblock In \emph{International Conference on Learning Representations
  (ICLR)}, 2022.

\bibitem[Heusel et~al.(2017)Heusel, Ramsauer, Unterthiner, Nessler, and
  Hochreiter]{heusel2017gans}
Martin Heusel, Hubert Ramsauer, Thomas Unterthiner, Bernhard Nessler, and Sepp
  Hochreiter.
\newblock Gans trained by a two time-scale update rule converge to a local nash
  equilibrium.
\newblock \emph{Conference on Neural Information Processing Systems (NeurIPS)},
  2017.

\bibitem[Jing et~al.(2022)Jing, Vincent, LeCun, and
  Tian]{jing2022understanding}
Li~Jing, Pascal Vincent, Yann LeCun, and Yuandong Tian.
\newblock Understanding dimensional collapse in contrastive self-supervised
  learning.
\newblock In \emph{International Conference on Learning Representations
  (ICLR)}, 2022.

\bibitem[Karpathy \& Fei-Fei(2015)Karpathy and Fei-Fei]{karpathysplit}
Andrej Karpathy and Li~Fei-Fei.
\newblock Deep visual-semantic alignments for generating image descriptions.
\newblock In \emph{Conference on Computer Vision and Pattern Recognition
  (CVPR)}, 2015.

\bibitem[Karras et~al.(2020)Karras, Laine, Aittala, Hellsten, Lehtinen, and
  Aila]{Karras2019stylegan2}
Tero Karras, Samuli Laine, Miika Aittala, Janne Hellsten, Jaakko Lehtinen, and
  Timo Aila.
\newblock Analyzing and improving the image quality of {StyleGAN}.
\newblock In \emph{Conference on Computer Vision and Pattern Recognition
  (CVPR)}, 2020.

\bibitem[Li et~al.(2023)Li, Zhu, Wen, and Yang]{li2023decap}
Wei Li, Linchao Zhu, Longyin Wen, and Yi~Yang.
\newblock {DECAP}: Decoding {CLIP} latents for zero-shot captioning.
\newblock In \emph{International Conference on Learning Representations
  (ICLR)}, 2023.

\bibitem[Liang et~al.(2022)Liang, Zhang, Kwon, Yeung, and Zou]{liang2022mind}
Weixin Liang, Yuhui Zhang, Yongchan Kwon, Serena Yeung, and James Zou.
\newblock Mind the gap: Understanding the modality gap in multi-modal
  contrastive representation learning.
\newblock In \emph{Conference on Neural Information Processing Systems
  (NeurIPS)}, 2022.

\bibitem[Lin(2004)]{lin2004rouge}
Chin-Yew Lin.
\newblock Rouge: A package for automatic evaluation of summaries.
\newblock In \emph{Workshop of Annual Meeting of the Association for
  Computational Linguistics (ACL Workshop)}, 2004.

\bibitem[Lin et~al.(2014)Lin, Maire, Belongie, Hays, Perona, Ramanan,
  Doll{\'a}r, and Zitnick]{mscoco}
Tsung-Yi Lin, Michael Maire, Serge Belongie, James Hays, Pietro Perona, Deva
  Ramanan, Piotr Doll{\'a}r, and C~Lawrence Zitnick.
\newblock Microsoft coco: Common objects in context.
\newblock In \emph{European Conference on Computer Vision (ECCV)}, 2014.

\bibitem[Loshchilov \& Hutter(2019)Loshchilov and Hutter]{adamw}
Ilya Loshchilov and Frank Hutter.
\newblock Decoupled weight decay regularization.
\newblock In \emph{International Conference on Learning Representations
  (ICLR)}, 2019.

\bibitem[Mokady et~al.(2021)Mokady, Hertz, and Bermano]{mokady2021clipcap}
Ron Mokady, Amir Hertz, and Amit~H Bermano.
\newblock Clipcap: Clip prefix for image captioning.
\newblock \emph{arXiv preprint arXiv:2111.09734}, 2021.

\bibitem[Nukrai et~al.(2022)Nukrai, Mokady, and
  Globerson]{nukrai-etal-2022-text}
David Nukrai, Ron Mokady, and Amir Globerson.
\newblock Text-only training for image captioning using noise-injected {CLIP}.
\newblock In \emph{Findings of Conference on Empirical Methods in Natural
  Language Processing (EMNLP Findings)}, 2022.

\bibitem[Papineni et~al.(2002)Papineni, Roukos, Ward, and
  Zhu]{papineni2002bleu}
Kishore Papineni, Salim Roukos, Todd Ward, and Wei-Jing Zhu.
\newblock Bleu: a method for automatic evaluation of machine translation.
\newblock In \emph{Annual Meeting of the Association for Computational
  Linguistics (ACL)}, 2002.

\bibitem[Radford et~al.(2019)Radford, Wu, Child, Luan, Amodei, and
  Sutskever]{radford2019language}
Alec Radford, Jeff Wu, Rewon Child, David Luan, Dario Amodei, and Ilya
  Sutskever.
\newblock Language models are unsupervised multitask learners.
\newblock \emph{OpenAI}, 2019.

\bibitem[Radford et~al.(2021)Radford, Kim, Hallacy, Ramesh, Goh, Agarwal,
  Sastry, Askell, Mishkin, Clark, et~al.]{radford2021learning}
Alec Radford, Jong~Wook Kim, Chris Hallacy, Aditya Ramesh, Gabriel Goh,
  Sandhini Agarwal, Girish Sastry, Amanda Askell, Pamela Mishkin, Jack Clark,
  et~al.
\newblock Learning transferable visual models from natural language
  supervision.
\newblock In \emph{International Conference on Machine Learning (ICML)}, 2021.

\bibitem[Ramesh et~al.(2021)Ramesh, Pavlov, Goh, Gray, Voss, Radford, Chen, and
  Sutskever]{ramesh2021zero}
Aditya Ramesh, Mikhail Pavlov, Gabriel Goh, Scott Gray, Chelsea Voss, Alec
  Radford, Mark Chen, and Ilya Sutskever.
\newblock Zero-shot text-to-image generation.
\newblock In \emph{International Conference on Machine Learning (ICML)}, 2021.

\bibitem[Ramesh et~al.(2022)Ramesh, Dhariwal, Nichol, Chu, and
  Chen]{ramesh2022hierarchical}
Aditya Ramesh, Prafulla Dhariwal, Alex Nichol, Casey Chu, and Mark Chen.
\newblock Hierarchical text-conditional image generation with clip latents.
\newblock \emph{arXiv preprint arXiv:2204.06125}, 2022.

\bibitem[Salimans et~al.(2016)Salimans, Goodfellow, Zaremba, Cheung, Radford,
  and Chen]{salimans2016improved}
Tim Salimans, Ian Goodfellow, Wojciech Zaremba, Vicki Cheung, Alec Radford, and
  Xi~Chen.
\newblock Improved techniques for training gans.
\newblock \emph{Conference on Neural Information Processing Systems (NeurIPS)},
  2016.

\bibitem[Sharma et~al.(2018)Sharma, Ding, Goodman, and
  Soricut]{sharma-etal-2018-conceptual}
Piyush Sharma, Nan Ding, Sebastian Goodman, and Radu Soricut.
\newblock Conceptual captions: A cleaned, hypernymed, image alt-text dataset
  for automatic image captioning.
\newblock In \emph{Annual Meeting of the Association for Computational
  Linguistics (ACL)}, 2018.

\bibitem[Shen et~al.(2022)Shen, Li, Tan, Bansal, Rohrbach, Chang, Yao, and
  Keutzer]{shen2022how}
Sheng Shen, Liunian~Harold Li, Hao Tan, Mohit Bansal, Anna Rohrbach, Kai-Wei
  Chang, Zhewei Yao, and Kurt Keutzer.
\newblock How much can {CLIP} benefit vision-and-language tasks?
\newblock In \emph{International Conference on Learning Representations
  (ICLR)}, 2022.

\bibitem[Su et~al.(2022)Su, Lan, Liu, Liu, Yogatama, Wang, Kong, and
  Collier]{su2022language}
Yixuan Su, Tian Lan, Yahui Liu, Fangyu Liu, Dani Yogatama, Yan Wang, Lingpeng
  Kong, and Nigel Collier.
\newblock Language models can see: plugging visual controls in text generation.
\newblock \emph{arXiv preprint arXiv:2205.02655}, 2022.

\bibitem[Szegedy et~al.(2016)Szegedy, Vanhoucke, Ioffe, Shlens, and
  Wojna]{szegedy2016rethinking}
Christian Szegedy, Vincent Vanhoucke, Sergey Ioffe, Jon Shlens, and Zbigniew
  Wojna.
\newblock Rethinking the inception architecture for computer vision.
\newblock In \emph{Conference on Computer Vision and Pattern Recognition
  (CVPR)}, 2016.

\bibitem[Tam et~al.(2023)Tam, Raffel, and Bansal]{tam2023simple}
Derek Tam, Colin Raffel, and Mohit Bansal.
\newblock Simple weakly-supervised image captioning via {CLIP}'s multimodal
  embeddings.
\newblock In \emph{Workshop of AAAI Conference on Artificial Intelligence (AAAI
  Workshop)}, 2023.

\bibitem[Tewel et~al.(2022)Tewel, Shalev, Schwartz, and Wolf]{tewel2022zerocap}
Yoad Tewel, Yoav Shalev, Idan Schwartz, and Lior Wolf.
\newblock Zerocap: Zero-shot image-to-text generation for visual-semantic
  arithmetic.
\newblock In \emph{Conference on Computer Vision and Pattern Recognition
  (CVPR)}, 2022.

\bibitem[Vedantam et~al.(2015)Vedantam, Lawrence~Zitnick, and
  Parikh]{vedantam2015cider}
Ramakrishna Vedantam, C~Lawrence~Zitnick, and Devi Parikh.
\newblock Cider: Consensus-based image description evaluation.
\newblock In \emph{Conference on Computer Vision and Pattern Recognition
  (CVPR)}, 2015.

\bibitem[Vendrow et~al.(2023)Vendrow, Jain, Engstrom, and
  Madry]{vendrow2023dataset}
Joshua Vendrow, Saachi Jain, Logan Engstrom, and Aleksander Madry.
\newblock Dataset interfaces: Diagnosing model failures using controllable
  counterfactual generation.
\newblock \emph{arXiv preprint arXiv:2302.07865}, 2023.

\bibitem[Wortsman et~al.(2022)Wortsman, Ilharco, Kim, Li, Kornblith, Roelofs,
  Lopes, Hajishirzi, Farhadi, Namkoong, et~al.]{wortsman2022robust}
Mitchell Wortsman, Gabriel Ilharco, Jong~Wook Kim, Mike Li, Simon Kornblith,
  Rebecca Roelofs, Raphael~Gontijo Lopes, Hannaneh Hajishirzi, Ali Farhadi,
  Hongseok Namkoong, et~al.
\newblock Robust fine-tuning of zero-shot models.
\newblock In \emph{Conference on Computer Vision and Pattern Recognition
  (CVPR)}, 2022.

\bibitem[Xu et~al.(2021)Xu, Ghosh, Huang, Okhonko, Aghajanyan, Metze,
  Zettlemoyer, and Feichtenhofer]{xu2021videoclip}
Hu~Xu, Gargi Ghosh, Po-Yao Huang, Dmytro Okhonko, Armen Aghajanyan, Florian
  Metze, Luke Zettlemoyer, and Christoph Feichtenhofer.
\newblock Videoclip: Contrastive pre-training for zero-shot video-text
  understanding.
\newblock In \emph{Conference on Empirical Methods in Natural Language
  Processing (EMNLP)}, 2021.

\bibitem[Xu et~al.(2016)Xu, Mei, Yao, and Rui]{xu2016msr-vtt}
Jun Xu, Tao Mei, Ting Yao, and Yong Rui.
\newblock Msr-vtt: A large video description dataset for bridging video and
  language.
\newblock In \emph{Conference on Computer Vision and Pattern Recognition
  (CVPR)}, 2016.

\bibitem[Yu et~al.(2022)Yu, Chung, Yun, Hessel, Park, Lu, Ammanabrolu, Zellers,
  Bras, Kim, et~al.]{yu2022multimodal}
Youngjae Yu, Jiwan Chung, Heeseung Yun, Jack Hessel, JaeSung Park, Ximing Lu,
  Prithviraj Ammanabrolu, Rowan Zellers, Ronan~Le Bras, Gunhee Kim, et~al.
\newblock Multimodal knowledge alignment with reinforcement learning.
\newblock \emph{arXiv preprint arXiv:2205.12630}, 2022.

\bibitem[Yuan et~al.(2021)Yuan, Chen, Chen, Codella, Dai, Gao, Hu, Huang, Li,
  Li, et~al.]{yuan2021florence}
Lu~Yuan, Dongdong Chen, Yi-Ling Chen, Noel Codella, Xiyang Dai, Jianfeng Gao,
  Houdong Hu, Xuedong Huang, Boxin Li, Chunyuan Li, et~al.
\newblock Florence: A new foundation model for computer vision.
\newblock \emph{arXiv preprint arXiv:2111.11432}, 2021.

\bibitem[Zhang et~al.(2022)Zhang, Jiang, Miura, Manning, and
  Langlotz]{zhang2020contrastive}
Yuhao Zhang, Hang Jiang, Yasuhide Miura, Christopher~D Manning, and Curtis~P
  Langlotz.
\newblock Contrastive learning of medical visual representations from paired
  images and text.
\newblock In \emph{Machine Learning for Healthcare (MLHC)}, 2022.

\bibitem[Zhang et~al.(2023)Zhang, HaoChen, Huang, Wang, Zou, and
  Yeung]{zhang2023diagnosing}
Yuhui Zhang, Jeff~Z HaoChen, Shih-Cheng Huang, Kuan-Chieh Wang, James Zou, and
  Serena Yeung.
\newblock Diagnosing and rectifying vision models using language.
\newblock In \emph{International Conference on Learning Representations
  (ICLR)}, 2023.

\bibitem[Zhou et~al.(2022{\natexlab{a}})Zhou, Li, Chen, Gao, and
  Xu]{zhou2022lafite2}
Yufan Zhou, Chunyuan Li, Changyou Chen, Jianfeng Gao, and Jinhui Xu.
\newblock Lafite2: Few-shot text-to-image generation.
\newblock \emph{arXiv preprint arXiv:2210.14124}, 2022{\natexlab{a}}.

\bibitem[Zhou et~al.(2022{\natexlab{b}})Zhou, Liu, Zhu, Yang, Chen, and
  Xu]{zhou2022shifted}
Yufan Zhou, Bingchen Liu, Yizhe Zhu, Xiao Yang, Changyou Chen, and Jinhui Xu.
\newblock Shifted diffusion for text-to-image generation.
\newblock \emph{arXiv preprint arXiv:2211.15388}, 2022{\natexlab{b}}.

\bibitem[Zhou et~al.(2022{\natexlab{c}})Zhou, Zhang, Chen, Li, Tensmeyer, Yu,
  Gu, Xu, and Sun]{zhou2022lafite}
Yufan Zhou, Ruiyi Zhang, Changyou Chen, Chunyuan Li, Chris Tensmeyer, Tong Yu,
  Jiuxiang Gu, Jinhui Xu, and Tong Sun.
\newblock Lafite: Towards language-free training for text-to-image generation.
\newblock In \emph{Conference on Computer Vision and Pattern Recognition
  (CVPR)}, June 2022{\natexlab{c}}.

\end{thebibliography}
\bibliographystyle{iclr2024_conference}

\newpage
\appendix
\section*{Overview of Appendix}
\label{sec:supp:overview}

In this appendix, we supplement related works and additional details of theory and experiments.

\begin{itemize}
    \item In Appendix~\ref{sec:related_works}, we provide detailed related works to contextualize our work in existing works of multi-modal contrastive learning and learning cross-modal tasks with uni-modal data.
    \item In Appendix~\ref{sec:appendix_proof}, we provide proofs of the two lemmas used in the main paper that reveal important properties of multi-modal contrastive learning.
    \item In Appendix~\ref{sec:empirical_verification}, we employ statistical methods to validate the proposed geometric structure on large pre-trained contrastive models.
    \item In Appendix~\ref{sec:appendix_algorithm}, we summarize the $C^3$ method into an algorithm.
    \item In Appendix~\ref{sec:appendix_i2t}, we provide additional experimental details and qualitative results of image captioning.
    \item In Appendix~\ref{sec:appendix_t2i}, we provide additional experimental details and qualitative results of text-to-image generation.
    \item In Appendix~\ref{sec:appendix_align}, we explain the importance of aligning embeddings from different modalities.
    \item In Appendix~\ref{sec:appendix_collapse_vs_corrupt}, we offer further discussions about the effectiveness of collapse vs corrupt.
    \item In Appendix~\ref{sec:appendix_collapse}, we offer further insights into dimensional collapse and connect it to the cone effect identified by~\citet{liang2022mind}.
\end{itemize}

\section{Related Works}
\label{sec:related_works}

\paragraph{Multi-modal contrastive learning and resulting geometry.} 

Multi-modal contrastive learning aims to create a shared representation for different modalities by attracting similar while repelling dissimilar concepts from different modalities during the optimization process~\citep{radford2021learning,zhang2020contrastive,xu2021videoclip,CLASP,yuan2021florence,girdhar2023imagebind}. Recent works such as CLIP~\citep{radford2021learning} have leveraged large-scale image-text data during pre-training, resulting in models that can build strong uni-modal and cross-modal applications. Connecting different modalities can be advantageous given the complementarity of different modalities. For example, connecting vision and language enables zero-shot categorization of visual objects~\citep{radford2021learning,yuan2021florence}, explanation of model prediction errors or internal representations~\citep{eyuboglu2021domino,hernandez2022natural}, diagnosis and rectification of vision models using language by composing different concepts~\citep{zhang2023diagnosing,vendrow2023dataset}, and learning cross-modal tasks with uni-modal data.

However, the geometry resulting from multi-modal contrastive learning has received limited study. Recent work~\citep{liang2022mind} found that there is a clear distinction between embeddings from different modalities in the shared representation space, which is referred to as the modality gap. \citet{liang2022mind} correctly attributed the gap to the joint effect of model initialization and optimization. However, they did not study the geometric property of the gap and their theory cannot explain the geometry as well. Their theory can only show there is a distributional difference between image and text embeddings. Subsequently, \citet{zhang2023diagnosing} studied the geometric properties of the modality gap and found that the gap can be empirically well approximated by a constant vector orthogonal to the image or text embedding subspace. However, this work did not provide an explanation for how this unique geometry arises.

In our work, we unify the observations from \citet{liang2022mind} and \citet{zhang2023diagnosing}, and contribute a formal formulation and theoretical explanation of the unique geometry resulting from multi-modal contrastive learning. Specifically, we show two important factors of the geometry of the multi-modal representation space: modality gap and alignment noise, where the modality gap is a constant vector orthogonal to the image and text embedding span, and the alignment noise can be approximated by a Gaussian distribution. The modality gap arises due to the interplay between the dimensional collapse in model initialization and the resulting collapsed gradient during optimization, whereas the alignment noise is due to the stable region of contrastive loss. These explanations provide a deeper understanding of multi-modal contrastive learning and resulting geometry. Moreover, this non-trivial geometry has important implications for building applications on top of a multi-modal representation space.

\paragraph{Learning cross-modal tasks with uni-modal data.}

Multi-modal data are less abundant and more expensive to collect than uni-modal data, making it ideal to learn cross-modal tasks with uni-modal data. The recent rise of multi-modal contrastive learning provides this possibility, as similar concepts from different modalities establish close connections during the optimization process. Many recent works have leveraged CLIP, which aligns images and text to achieve image-to-text captioning with text-only data and text-to-image generation with image-only data. These methods assume and explore how to interchangeably use image embeddings and text embeddings resulting from multi-modal contrastive learning. They achieve better performance than prior methods that did not leverage a multi-modal representation space, outperforming ZeroCap~\citep{tewel2022zerocap}, MAGIC~\citep{su2022language}, and ESPER~\citep{yu2022multimodal} on captioning, and DALL-E~\citep{ramesh2021zero} and CogView~\citep{ding2021cogview} on image generation. However, due to the peculiar geometry that arises from multi-modal contrastive learning, paired image embeddings and text embeddings are not collapsed to the same point, making it non-trivial to substitute one for the other. Therefore, researchers are proposing different methods to tackle this problem.

For image-to-text captioning, WS-ClipCap~\citep{li2023decap} was the first work to leverage CLIP's text embeddings to decode text during training, finding that training to decode a paraphrased version of the text from the corresponding CLIP's text embedding leads to significantly better performance (WS-ClipCap-Multi). That is, for datasets where a single image is paired with multiple captions, it is better to decode a caption by feeding the text embedding obtained from one of the other captions corresponding to the same image. Despite proposing the method, WS-ClipCap failed to explain why it works and that on a high level, this paraphrased decoding can be viewed as adding noise. Decap~\citep{li2023decap} maintains a memory of image embeddings and converts image embeddings to text embeddings by using a weighted average of the most similar image embeddings. They then feed the decoder with the converted embedding to decode the same text. While their method outperforms the baseline CLIPRe, which directly retrieves the most similar captions based on image embeddings, it underperforms WS-ClipCap-Multi, despite proposing a more complex method to tackle the modality gap. CapDec~\citep{nukrai-etal-2022-text} found that directly adding Gaussian noise to the embedding and then feeding it into the decoder leads to comparably better performance. This method corresponds exactly to the corrupt stage in our method. CapDec attributes its success to the hypothesis that adding Gaussian noise closes the modality gap, however, this intuition is inaccurate.

For text-to-image generation, LAFITE~\citep{zhou2022lafite} was the first work to use CLIP with uni-modal images only. Their approach can be seen as the inverse version of CapDec, where they add Gaussian noise to the image embedding and decode the same image. They also incorrectly suggest that adding Gaussian noise closes the modality gap. In contrast, DALLE-2~\citep{ramesh2022hierarchical} uses a complex and heavily trained prior network to convert CLIP text embeddings to CLIP image embeddings but did not provide clear evidence of why the prior network is necessary and its effectiveness. Corgi~\citep{zhou2022shifted} modifies the prior network by restricting its starting point, resulting in improved performance.

Despite the plethora of complex tricks proposed, there is no consensus on the best approach to interchangeably use image and text embeddings, and none of them reveal fundamental reasons or perform systematic ablations. 
Our work unifies previous approaches by first analyzing the multi-modal representation space and how it arises. Based on the geometric relation, we provide a straightforward method to tackle embedding interchangeability and demonstrate that our method achieves state-of-the-art results on cross-modal tasks when only training on uni-modal data. 

\section{Theory Proof}
\label{sec:appendix_proof}

In this section, we provide proofs of the two lemmas used in the main paper that reveal important properties of multi-modal contrastive learning.

\subsection{Proof of Lemma 1}
\label{sec:prop_gradient}

\begin{lemma}
\textup{\textbf{(Gradients in Contrastive Optimization)}}\\
With the mild assumption of equal presence of $n$ images and texts with $p(x_i)=p(y_i)=1/n$, optimizing the multi-modal contrastive loss $\gL = -\frac{1}{2n} \sum_{i=1}^n \big(\log \frac{\exp(\ve_{x_i} \cdot \ve_{y_i} / \tau)}{\sum_{j=1}^n \exp(\ve_{x_i} \cdot \ve_{y_j}  / \tau)} + \log \frac{\exp(\ve_{x_i} \cdot \ve_{y_i} / \tau)}{\sum_{j=1}^n \exp(\ve_{x_j} \cdot \ve_{y_i}  / \tau)}\big)$ yields the following gradients:
\begin{align*}
\nabla_{\ve_{x_i}} \gL &= \lambda \sum_{j=1}^n \alpha_{y_j} (\ve_{y_j} - \ve_{y_i}) \\
\nabla_{\ve_{y_i}} \gL &= \lambda \sum_{j=1}^n \alpha_{x_j} (\ve_{x_j} - \ve_{x_i})
\end{align*}

\noindent where $\lambda=1 / (2n\tau)$, $\alpha_{x_j} = p(x_j|y_i)+p(y_i|x_j)$, 
$\alpha_{y_j} = p(y_j|x_i)+p(x_i|y_j)$, 
$p(x_i|y_j)=\frac{\exp(\ve_{x_i} \cdot \ve_{y_j} / \tau)} {\sum_{k=1}^n \exp(\ve_{x_k} \cdot \ve_{y_j} / \tau)}$, 
$p(y_i|x_j) = \frac{\exp(\ve_{y_i} \cdot \ve_{x_j} / \tau)}{\sum_{k=1}^n \exp(\ve_{y_k} \cdot \ve_{x_j} / \tau)}$, and $\tau$ is temperature.
\end{lemma}

\begin{proof}[Proof of Lemma~\ref{prop:gradient}]

We first prove $\forall k, \sum_{i=1}^n p(x_k|y_i) = 1$ using Bayes' theorem: 
\begin{align*}
\sum_{i=1}^n p(x_k|y_i) = \sum_{i=1}^n \frac{p(y_i|x_k)p(x_k)}{p(y_i)} = \sum_{i=1}^n \frac{p(y_i|x_k) (1/ n)} {(1/n)} = 1
\end{align*}

Then, we can prove the lemma using chain rule:
\begin{align*}
&\nabla_{\ve_{x_k}} \gL \\
&= \nabla_{\ve_{x_k}} \big[-\frac{1}{2n} \sum_{i=1}^n \big( \log \frac{\exp(\ve_{x_i} \cdot \ve_{y_i} / \tau)}{\sum_{j=1}^n \exp(\ve_{x_i} \cdot \ve_{y_j}  / \tau)} + \log \frac{\exp(\ve_{x_i} \cdot \ve_{y_i} / \tau)}{\sum_{j=1}^n \exp(\ve_{x_j} \cdot \ve_{y_i}  / \tau)}\big) \big] \\
&= -\frac{1}{2n} \nabla_{\ve_{x_k}} \big[ \sum_{i=1}^n 2 \ve_{x_i} \cdot \ve_{y_i} / \tau - \sum_{i=1}^n \log ({\sum_{j=1}^n \exp(\ve_{x_i} \cdot \ve_{y_j}  / \tau)})  - \sum_{i=1}^n \log ({\sum_{j=1}^n \exp(\ve_{x_j} \cdot \ve_{y_i}  / \tau)}) \big] \\
&= -\frac{1}{2n} \big[ 2 \ve_{y_k} / \tau -  \nabla_{\ve_{x_k}} \log ({\sum_{j=1}^n \exp(\ve_{x_k} \cdot \ve_{y_j}  / \tau)})  - \sum_{i=1}^n \nabla_{\ve_{x_k}}  \log ({\sum_{j=1}^n \exp(\ve_{x_j} \cdot \ve_{y_i}  / \tau)}) \big] \\
&= -\frac{1}{2n} \big[ 2 \ve_{y_k} / \tau -  \sum_{i=1}^n \frac{\exp(\ve_{x_k} \cdot \ve_{y_i}/ \tau)} { {\sum_{j=1}^n \exp(\ve_{x_k} \cdot \ve_{y_j}  / \tau)}} \ve_{y_i} / \tau  - \sum_{i=1}^n \frac{\exp(\ve_{x_k} \cdot \ve_{y_i}/ \tau)}{\sum_{j=1}^n \exp(\ve_{x_j} \cdot \ve_{y_i}  / \tau)} \ve_{y_i} / \tau  \big] \\
&= -\frac{1}{2n\tau} \big[ 2 \ve_{y_k}  -  \sum_{i=1}^n p(y_i|x_k) \ve_{y_i}   - \sum_{i=1}^n p(x_k|y_i) \ve_{y_i}   \big] \\
&= \frac{1}{2n\tau}  \big[ \sum_{i=1}^n p(y_i|x_k) (\ve_{y_i} - \ve_{y_k})   + \sum_{i=1}^n p(x_k|y_i) (\ve_{y_i} - \ve_{y_k})   - (1-\sum_{i=1}^n p(x_k|y_i) ) \ve_{y_k}  \big ] \\
&= \frac{1}{2n\tau}  \big[ \sum_{i=1}^n (p(y_i|x_k) + p(x_k|y_i)) (\ve_{y_i} - \ve_{y_k}) ] \\
&= \lambda \sum_{i=1}^n \alpha_{y_i} (\ve_{y_i} - \ve_{y_k})  \\
\end{align*}

Similarly, we can get $\nabla_{\ve_{y_k}} \gL = \lambda \sum_{i=1}^n \alpha_{x_i} (\ve_{x_i} - \ve_{x_k})$, which finishes the proof.
\end{proof}

\subsection{Proof of Lemma 2}
\label{sec:prop_region}

\begin{lemma}
\textup{\textbf{(Stable Region Controlled by Temperature)}} \\
We consider a single term in the multi-modal contrastive loss
$\gL_i=-\log \frac{\exp(\ve_{x_i} \cdot \ve_{y_i} / \tau)}{\sum_{j=1}^n \exp(\ve_{x_i} \cdot \ve_{y_j} / \tau)}$. We define the \emph{margin} $r = \ve_{x_i} \cdot \ve_{y_i} - \max_{j \ne i} \ve_{x_i} \cdot \ve_{y_j}$ as the measure of the similarity difference between the matched pair and the hardest negative pair. When $r$ exceeds a threshold given below, $\gL_i$ falls below a small pre-set value $\delta$, where we assume optimization ends:
\begin{align*}
r \ge \tau \log \frac{o(\tau)}{\exp (\delta) - 1},
\end{align*}
where $o(\tau)$ is a monotonically increasing function of temperature $\tau$ that satisfies $1 < o(\tau) < n$. Therefore, the required margin $r$ is monotonically increasing with $\tau$.
\end{lemma}

\begin{proof}[Proof of Lemma~\ref{prop:region}]

We first prove $\sum_i \exp(t_i / \tau) \le o(\tau) \exp( \max_i t_i / \tau)$, where $o(\tau)$ is a monotonically increasing function of $\tau$ that satisfies $1 < o(\tau) < n$. Let us denote $m = \arg\max_i t_i$ (no tie), we have:
\begin{align*}
& \sum_i \exp(t_i / \tau) = \exp( t_m / \tau) (1+\sum_{i\ne m} \exp((t_i - t_m) / \tau))
\end{align*}
Let us denote $o'(\tau) = 1+\sum_{i\ne m} \exp((t_i - t_m) / \tau)$, since $\forall i \ne m, t_i - t_m < 0$, we have $0 < \exp((t_i - t_m) / \tau) < 1$, so $0 < \sum_{i\ne m} \exp((t_i - t_m) / \tau) < n-1$, therefore $1 < o'(\tau) < n$. Moreover, $\exp((t_i - t_m) / \tau)$ monotonically increases with $\tau$, therefore $o'(\tau)$ is a monotonically increasing function of $\tau$. We can denote $o(\tau) = \ceil{o'(\tau)}$.

Based on this, we have:
\begin{align*}
& \gL_i \\
&=-\log \frac{\exp(\ve_{x_i} \cdot \ve_{y_i} / \tau)}{\sum_{j=1}^n \exp(\ve_{x_i} \cdot \ve_{y_j} / \tau)} \\
&= -\log \frac{1}{1 + \sum_{j\ne i} \exp((\ve_{x_i} \cdot \ve_{y_j} - \ve_{x_i} \cdot \ve_{y_i}) / \tau)} \\
&= \log (1 + \sum_{j\ne i} \exp((\ve_{x_i} \cdot \ve_{y_j} - \ve_{x_i} \cdot \ve_{y_i}) / \tau)) \\
&\le \log (1 + o(\tau) \max_{j\ne i} \exp((\ve_{x_i} \cdot \ve_{y_j} - \ve_{x_i} \cdot \ve_{y_i}) / \tau)) \\
&= \log (1 + o(\tau) \exp(-r / \tau))
\end{align*}
Suppose $\gL_i \le \log (1 + o(\tau) \exp(-r / \tau)) \le \delta$, we have $r \ge \tau \log \frac{o(\tau)}{\exp (\delta) - 1}$, which finishes the proof.
\end{proof}

\section{Empirical Verification of Geometry}
\label{sec:empirical_verification}

In Section~\ref{sec:theory}, we established a theoretical framework for the multi-modal contrastive representation space geometry. We prove that after contrastive learning, we have $\ve_x - \ve_y = \vc_\perp + \vepsilon$, where $\ve_x$ and $\ve_y$ are $\ell_2$-normalized embeddings of a paired image $x$ and text $y$, $\vc_\perp$ is a constant vector representing the \emph{modality gap} and is orthogonal to the image and text embedding span, and $\vepsilon \sim \gN(\bm{0}, \sigma^2 \bm{I})$ is a random Gaussian vector representing the \emph{alignment noise}. Here we employ statistical methods to validate the proposed geometric structure on large pre-trained contrastive models. Figure~\ref{fig:gap_definition} visualizes all the definitions introduced in this section.

To analyze the modality gap $\vc_\perp$, we need to first find a way to isolate this vector, because it is entangled with alignment noise $\vepsilon$ in $\ve_x - \ve_y$. Since $\vepsilon$ is Gaussian, averaging multiple instances of $\ve_x - \ve_y$ should neutralize this noise, and thus isolating $\vc_\perp$. Therefore, we randomly group every 100 image-text pairs into a group $i$. We define embedding difference for pair $j$ in group $i$ as $\vd_j^{(i)} = \ve_{x_j}^{(i)} - \ve_{y_j}^{(i)}$. This difference includes both the modality gap and alignment noise. By computing the expected value, $\vd^{(i)} = \E_j [\vd_j^{(i)}]$, we effectively eliminate the noise component and leave the modality gap.

To demonstrate that the modality gap is a constant vector, we analyze its magnitude and direction across different groups. The distribution of $\|\vd^{(i)}\|_2$ (Figure~\ref{fig:gap_result} (Gap Length)) shows that the gap has a near constant length of 0.83, while the 0.99 mean of $\cos(\vd^{(i)}, \vd^{(j)})$ (Figure~\ref{fig:gap_result} (Gap Direction)) confirms that the gap has the same direction. These findings collectively validate that the modality gap is a constant vector.

Next, we establish the orthogonality of the modality gap to the embedding spans. By demonstrating its orthogonality to the image embedding span, we implicitly confirm its orthogonality to the text embedding span, given that the modality gap is a constant vector. We define the embedding difference of image $j$ and $k$ as $\vr_{j,k}^{(i)} = \ve_{x_j}^{(i)} - \ve_{x_k}^{(i)}$. We observe the distribution of $\cos(\vd^{(i)},\vr_{j,k}^{(i)})$ shown in Figure~\ref{fig:gap_result} (Gap Orthogonality) has zero mean with a small standard deviation 0.06, which demonstrates the gap's orthogonality to embedding spans.

Finally, we address the alignment noise. We define $\vepsilon_j^{(i)} = \vd_j^{(i)} - \vd^{(i)}$, which eliminates the gap and only leaves the noise. We first demonstrate its zero-mean nature $\E_{i,j} [\vepsilon_j^{(i)}] = \bm{0}$ through the distribution of each dimension's mean. We see the mean of each dimension of the noise vectors is bounded between -1e-8 and 1e-8  (Figure~\ref{fig:gap_result} (Noise Mean)).
Furthermore, we show in Figure~\ref{fig:gap_result} (Noise Direction) that the distribution of $\cos(\vepsilon_{j}^{(i)}, \vepsilon_{k}^{(i)})$ has zero mean and small standard deviation 0.10, indicating that noise has random directions as every two noises are likely to be orthogonal. These findings affirm that the alignment noise can be approximated by a Gaussian noise $\gN(\bm{0}, \sigma^2 \bm{I})$.

In summary, our empirical analyses support the proposed geometric model of the multi-modal contrastive representation space, with a constant modality gap vector orthogonal to embedding spans and an alignment noise characterizable as zero-mean Gaussian.

In Figure~\ref{fig:gap_result_more}, we show that our analyses above fully apply to other modalities, datasets, and contrastive embedding spaces, such as image-caption (MS-COCO), audio-caption (Clotho), and video-caption (MSR-VTT) on ImageBind embeddings.

\begin{figure}[htbp]
    \centering
    \includegraphics[width=\linewidth]{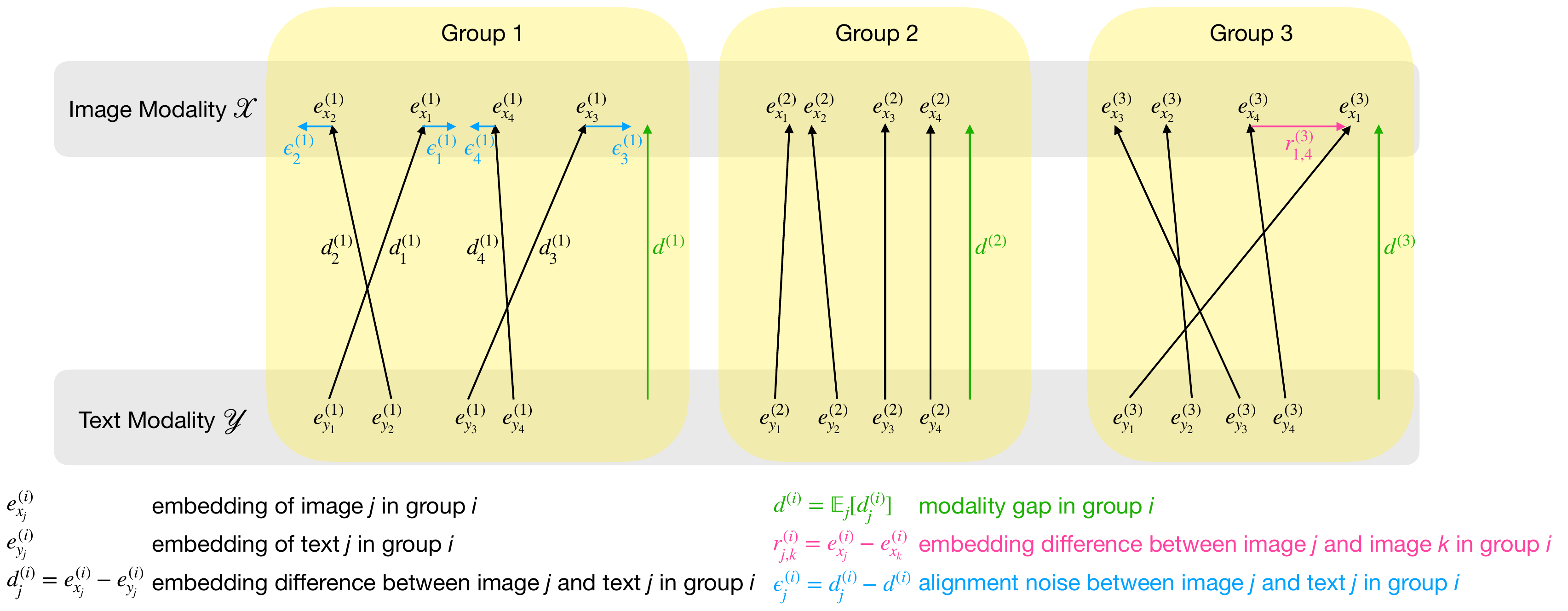}
    \vspace{-2em}
    \caption{\emph{Visualization of the multi-modal contrastive representation space and various definitions introduced in Appendix~\ref{sec:empirical_verification}.}}
    \label{fig:gap_definition}
\end{figure}

\begin{figure}[htbp]
    \centering
    \includegraphics[width=\linewidth]{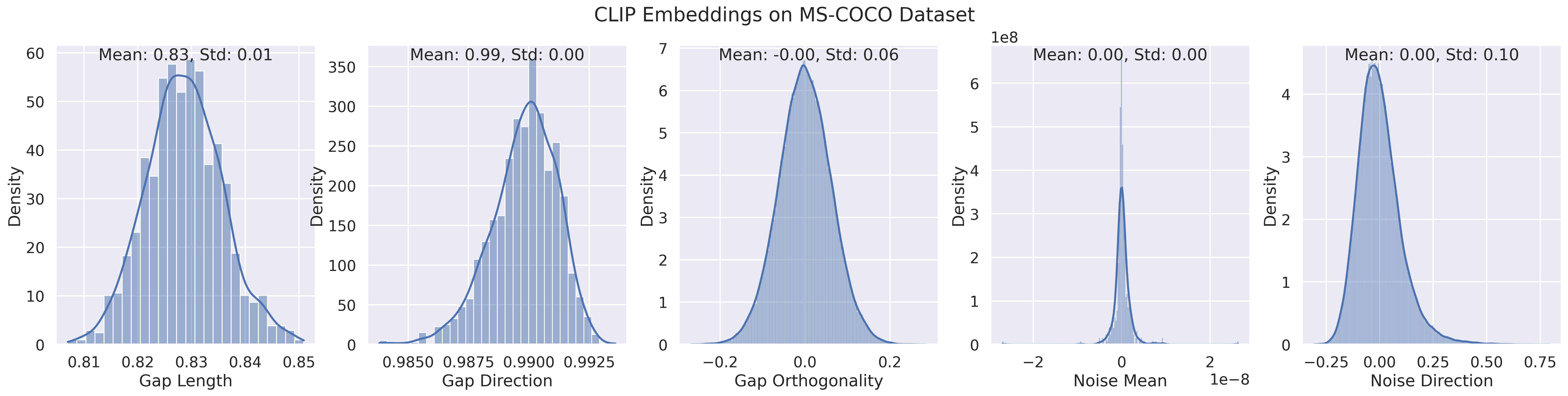}
    \vspace{-2em}
    \caption{\emph{Empirical verification of the multi-modal contrastive representation space geometry.} The modality gap approximates a constant vector, indicated by the gap length and direction distributions. The modality gap is orthogonal to the span of embeddings from two modalities, indicated by the gap orthogonality distributions. The alignment noise can be approximated by zero-mean Gaussian, indicated by the noise mean and direction distributions.}
    \label{fig:gap_result}
\end{figure}

\begin{figure}[htbp]
    \centering
    \includegraphics[width=\linewidth]{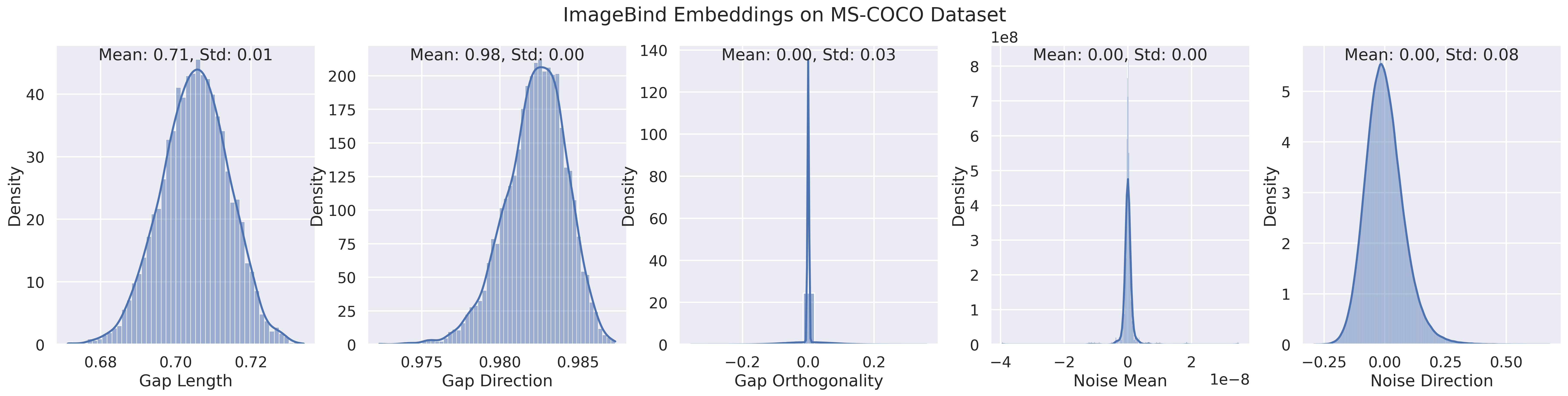}
    \includegraphics[width=\linewidth]{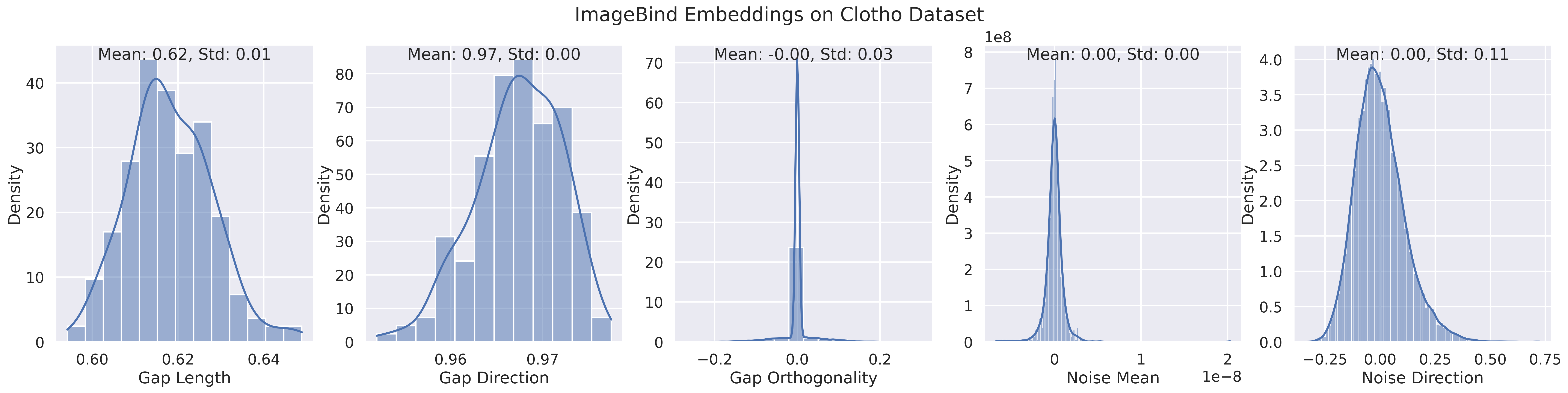}
    \includegraphics[width=\linewidth]{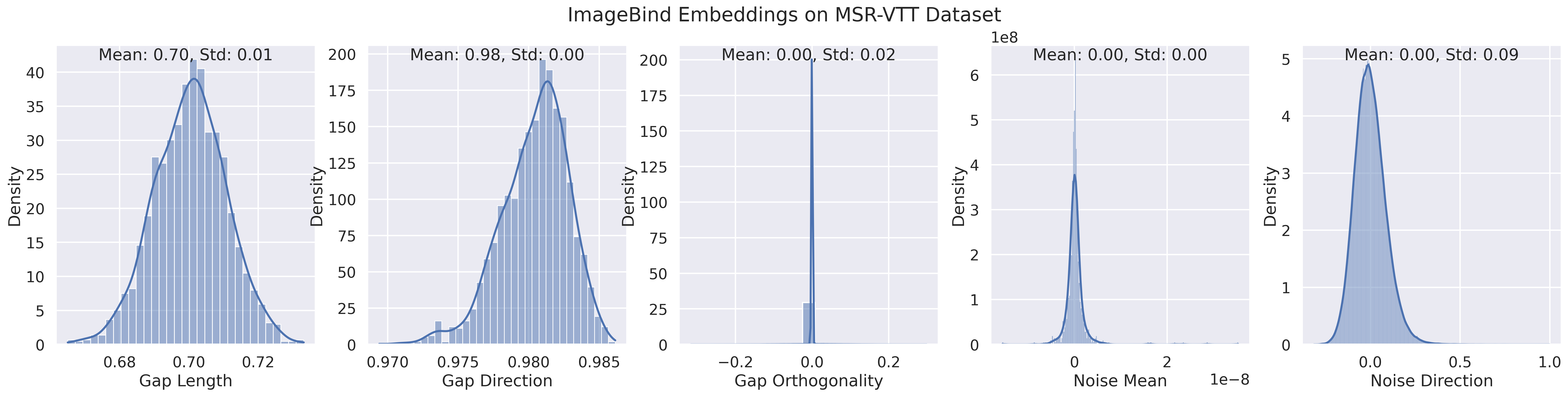}
    \vspace{-2em}
    \caption{\emph{Generalization of Figure~\ref{fig:gap_result} to other modalities and multi-modal contrastive embedding spaces, including image-caption (MS-COCO), audio-caption (Clotho), and video-caption (MSR-VTT) on ImageBind embeddings.}}
    \label{fig:gap_result_more}
\end{figure}

\section{\texorpdfstring{$C^3$}{C3} Algorithm}
\label{sec:appendix_algorithm}

In this section, we summarize the $C^3$ algorithm as follows:

\begin{algorithm}
\caption{$C^3$ algorithm}
\label{alg:example}
\begin{algorithmic}[1]
\Require unpaired uni-modal dataset $\gX$, $\gY$, fixed encoders $f_\gX: \gX \mapsto \R^d$, $f_\gY: \gY \mapsto \R^d$ obtained from multi-modal contrastive learning, trainable decoder $g: \R^d \mapsto \gY$, noise level $\sigma$

\State $\bar{\ve}_x \gets \sum_{x \in \gX} \frac{1}{|\gX|} f_\gX(x)$
\State $\bar{\ve}_y \gets \sum_{y \in \gY} \frac{1}{|\gY|} f_\gY(y)$

\\

\Function{train}{$y$, $f_\gY$, $g$}
\State $\ve_y \gets f_\gY(y)$
\State $\vepsilon \gets \gN(\bm{0}, \sigma^2 \bm{I})$
\State $\hat{y} \gets g(\ve_y - \bar{\ve}_y + \vepsilon)$
\State \Return $\gL(\hat{y}, y)$
\EndFunction

\\
\Function{test}{$x$, $f_\gX$, $g$}
\State $\ve_x \gets f_\gX(x)$
\State $\hat{y} \gets g(\ve_x - \bar{\ve}_x)$
\State \Return $\hat{y}$
\EndFunction

\end{algorithmic}
\label{alg:c3}
\end{algorithm}

\section{Image Captioning}
\label{sec:appendix_i2t}

In this section, we provide additional experimental details and qualitative results of image captioning.

\subsection{Experimental Setup}

\paragraph{Model.} We employ the ClipCap model architecture~\citep{mokady2021clipcap}, which utilizes the CLIP ViT-B/32~\citep{radford2021learning} as the image encoder $f_\gX$, and a mapping network with a pre-trained GPT-2~\citep{radford2019language} as the decoder $g$. The mapping network is designed to handle the difference in embedding dimensions between CLIP (512-$d$) and GPT-2 (768-$d$) and to generate a ``prefix'' as input to GPT-2. It is implemented as a lightweight MLP with a single hidden layer, which transforms the CLIP embedding into prefix embeddings for GPT-2 to generate captions. During training, we fix the CLIP encoder to maintain the connection between image and text embeddings and only update the decoder, which includes the mapping network and GPT-2.

\paragraph{Data.} To train and evaluate our model, we use the MS-COCO image-caption dataset~\citep{mscoco}. We adopt the widely-used data split~\citep{karpathysplit}, which consists of a training set of approximately 113K images, and validation and test sets of 5K images each, where each image has 5 ground truth captions.

\paragraph{Evaluation.} We evaluate our model using various commonly-used image captioning metrics, including BLEU~\citep{papineni2002bleu}, METEOR~\citep{banerjee2005meteor}, ROUGE~\citep{lin2004rouge}, CIDEr~\citep{vedantam2015cider}, and SPICE~\citep{anderson2016spice}. These metrics measure the lexical and semantic similarities between the generated captions and the ground-truth captions.

\paragraph{Setup.} We train our model for text reconstruction using the MS-COCO captions only. Following the $C^3$ algorithm, we extract the text embedding from the CLIP text encoder $f_\gY$ and apply the collapse operation (removing the pre-computed mean) and corrupt operation (adding Gaussian noise). After pre-training, we evaluate our model in the cross-modal setting. We replace the encoder with the CLIP image encoder and decode captions from image embeddings. Since no image is seen during pre-training, we refer to this evaluation setting as \emph{image-free zero-shot} evaluation. Additionally, we fine-tune our model on different amounts of image-caption pairs and evaluate its performance. We refer to this evaluation setting as \emph{semi-supervised} evaluation. During both the pre-training and fine-tuning stages, we train the model for 10 epochs with a batch size of 40, a learning rate of 2e-5, and AdamW~\citep{adamw} optimizer with a linear warmup of 5K steps. We use early stopping on the validation set and report the test set performance.

\subsection{Qualitative Examples}

\begin{figure*}[htbp]
    \centering
    \includegraphics[width=0.96\linewidth]{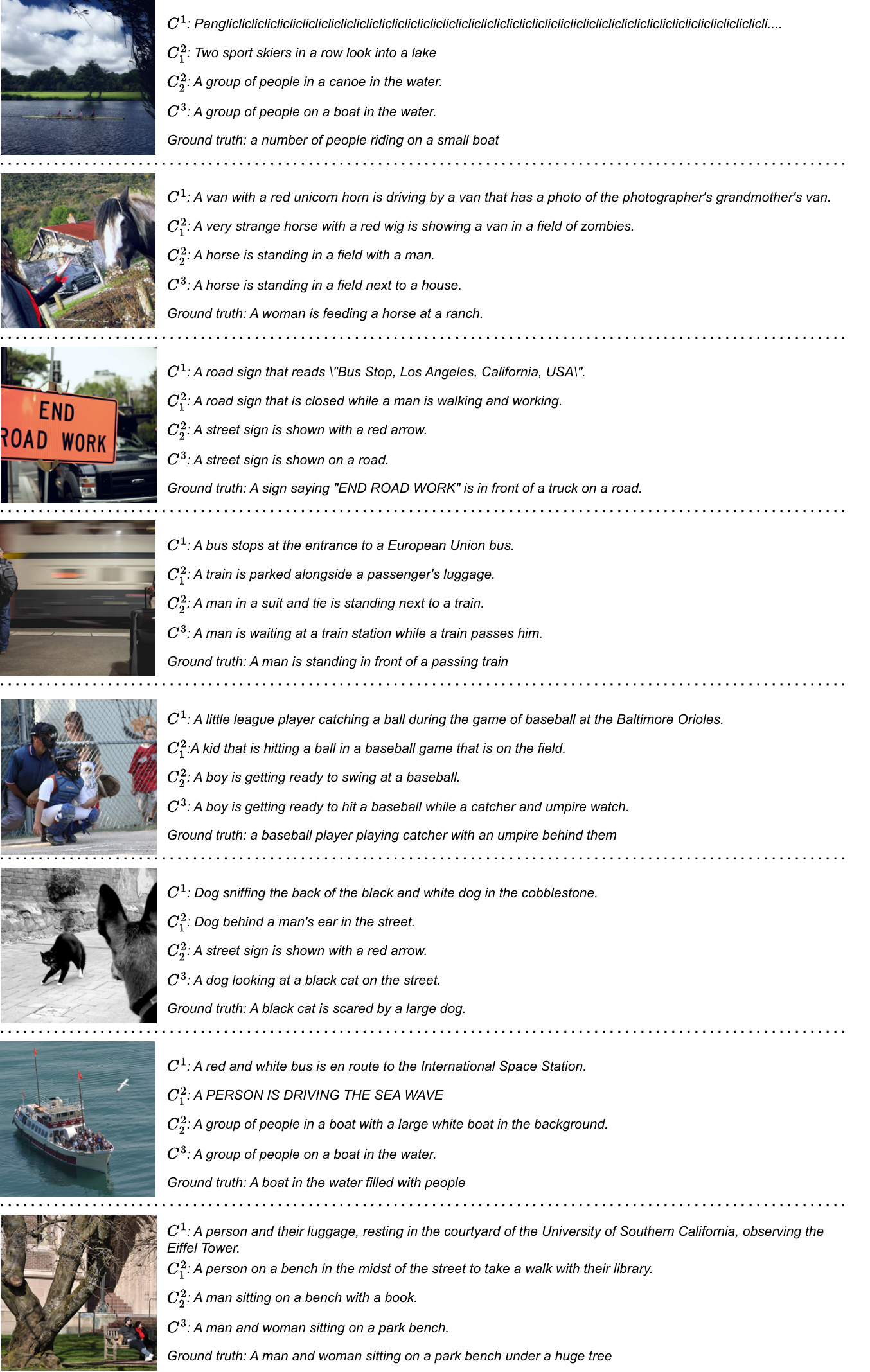}
    \vspace{-1em}
    \caption{\emph{Qualitative examples of image-to-text captioning on the MS-COCO dataset.} $C^1$ generates captions that are highly repetitive and/or nonsensical. $C^2_1$ and $C^2_2$ generates captions that are more fluent, but contain some hallucinations. $C^3$ generates captions that are more correct and concise with no extraneous details. }
    \label{fig:i2t}
\end{figure*}

We provide qualitative results for image captioning in Appendix Figure~\ref{fig:i2t}, which helps us better understand the improvements of each component of $C^3$. We observe that:
\begin{itemize}
    \item $C^1$ generates captions that are \textbf{highly repetitive and/or nonsensical}.
    \item $C^2_1$ and $C^2_2$ generates captions that are \textbf{more fluent, but contain some hallucinations}.
    \item $C^3$ generates captions that are \textbf{more correct and concise with no extraneous details}.
\end{itemize}

\subsection{Quantitative Results}

\begin{table*}[htbp]
\centering
\scriptsize

\begin{tabular}{l|ccc|cccccc}
\toprule

\textbf{Method} & \textbf{Connect} & \textbf{Collapse} & \textbf{Corrupted} & \textbf{BLEU-1} & \textbf{BLEU-4} & \textbf{METEOR} & \textbf{ROUGE-L} & \textbf{CIDEr} & \textbf{SPICE} \\
\midrule
\multicolumn{10}{c}{1\% Fine-tuning} \\
ClipCap & - & - & - & 64.5 & 21.2 & 21.2 & 47.6 & 71.1 & 14.7 \\ 
$C^1$ & \ding{51} & \ding{55} & \ding{55} & 66.5 & 22.9 & 22.6 & 48.7 & 78.5 & 15.9 \\
$C^2_1$ & \ding{51} & \ding{51} & \ding{55} & 66.8 & 23.0 & 22.6 & 48.7 & 78.5 & 15.9 \\ 
$C^2_2$ & \ding{51} & \ding{55} & \ding{51} & 71.7 & 28.8 & 25.3 & 52.7 & 96.3 & 18.3 \\
$C^3$ (Ours) & \ding{51} & \ding{51} & \ding{51} & 71.9 & 28.9 & 25.3 & 52.7 & 96.4 & 18.3 \\
\midrule

\multicolumn{10}{c}{5\% Fine-tuning} \\
ClipCap & - & - & - & 70.1 & 26.4 & 24.0 & 51.3 & 89.4 & 17.3 \\ 
$C^1$ & \ding{51} & \ding{55} & \ding{55} & 70.6 & 27.0 & 24.5 & 51.6 & 91.6 & 17.8 \\
$C^2_1$ & \ding{51} & \ding{51} & \ding{55} & 70.7 & 27.1 & 24.6 & 51.7 & 92.1 & 17.9 \\ 
$C^2_2$ & \ding{51} & \ding{55} & \ding{51} & 72.1 & 29.3 & 25.6 & 53.1 & 98.2 & 18.7 \\
$C^3$ (Ours) & \ding{51} & \ding{51} & \ding{51} & 72.3 & 29.4 & 25.7 & 53.1 & 98.8 & 18.8 \\
\midrule

\multicolumn{10}{c}{25\% Fine-tuning} \\
ClipCap & - & - & - & 73.0 & 30.0 & 26.0 & 53.7 & 101.7 & 19.1 \\ 
$C^1$ & \ding{51} & \ding{55} & \ding{55} & 73.2 & 30.4 & 26.1 & 53.9 & 102.8 & 19.4 \\
$C^2_1$ & \ding{51} & \ding{51} & \ding{55} & 73.1 & 30.2 & 26.1 & 53.8 & 102.4 & 19.3 \\ 
$C^2_2$ & \ding{51} & \ding{55} & \ding{51} & 73.2 & 30.9 & 26.3 & 54.2 & 103.1 & 19.3 \\
$C^3$ (Ours) & \ding{51} & \ding{51} & \ding{51} & 73.4 & 31.1 & 26.3 & 54.3 & 103.6 & 19.4 \\
\midrule

\multicolumn{10}{c}{100\% Fine-tuning} \\
ClipCap & - & - & - & 74.0 & 31.5 & 26.8 & 54.7 & 106.6 & 19.9 \\
$C^1$ & \ding{51} & \ding{55} & \ding{55} & 74.6 & 32.4 & 27.2 & 55.2 & 109.7 & 20.3 \\
$C^2_1$ & \ding{51} & \ding{51} & \ding{55} & 74.6 & 32.3 & 27.1 & 55.2 & 109.2 & 20.1 \\ 
$C^2_2$ & \ding{51} & \ding{55} & \ding{51} & 74.1 & 32.2 & 27.1 & 55.2 & 108.8 & 20.2 \\
$C^3$ (Ours) & \ding{51} & \ding{51} & \ding{51} & 74.0 & 32.2 & 27.1 & 55.2 & 108.9 & 20.2 \\ 
\bottomrule
\end{tabular}

\caption{\emph{Image-to-text captioning results in the low data regime.} When paired multi-modal data are limited, our approach that leverages uni-modal data for pre-training leads to substantial improvements compared to the purely supervised method (ClipCap). This is the table version used to reproduce Figure 5 in the main paper. Results are averaged over three random seeds for 1-25\% fine-tuning to reduce the effect of randomness.}
\label{tab:i2t-new}
\end{table*}

We include Appendix Table~\ref{tab:i2t-new}, which is the table version to produce Figure 5 in the main text. In this table, we compare the fully supervised ClipCap~\citep{mokady2021clipcap} to an ablation of components in our method $C^3$, and demonstrate the effectiveness of learning cross-modal tasks with uni-modal data. Results are averaged over three random seeds for 1-25\% fine-tuning to reduce the effect of randomness.

\section{Text-to-Image Generation}
\label{sec:appendix_t2i}

In this section, we provide additional experimental details and qualitative results of text-to-image generation.

\subsection{Experimental Setup} 

\paragraph{Model.} We use the LAFITE~\citep{zhou2022lafite} model for image generation, which uses the CLIP ViT-B/32~\citep{radford2021learning} as the text encoder $f_\gX$, and an adapted version of the unconditional generator from StyleGAN2~\citep{Karras2019stylegan2} as the decoder $g$. LAFITE's generator is adversarially trained alongside a discriminator with an additional contrastive objective to align the generator's representation space to that of CLIP's. As with image captioning, during training, we fix the CLIP encoder to maintain the connection between image and text embeddings, and only update the decoder, which in this case, involves updating both the generator and discriminator.

\paragraph{Data.} Same as image-to-text captioning, we train and evaluate the model on the MS-COCO dataset~\citep{mscoco}. We use the same pre-processed data and data split provided in the LAFITE official code repository, comprised of 82K training images and 40K validation images with 5 captions per image.

\paragraph{Evaluation.} We evaluate our model using the widely-used image generation metric Frechet Inception Distance (FID)~\citep{heusel2017gans}. This metric measure the realism of generated images by computing feature similarity between those of generated images and those of ground-truth images, where the features are derived from the pre-trained Inception-v3 model~\citep{szegedy2016rethinking}. We also report Inception Score (IS)~\citep{salimans2016improved}, which is similar to FID. Similarly to LAFITE, FID/IS scores are computed based on 50K generated images, using captions that are randomly sampled from the validation set.

\paragraph{Setup.} We train the model for at most 750,000 steps with a batch size of 16, a learning rate of 2.5e-3, and Adam~\citation{2015-kingma} optimizer. We initialize our model with the official pre-trained weights from LAFITE~\citep{zhou2022lafite} that was trained on CC3M~\citep{sharma-etal-2018-conceptual}. We report the validation set performance with the lowest FID score. 

\subsection{Qualitative Examples}

\begin{figure*}[htbp]
    \centering
    \includegraphics[width=\linewidth]{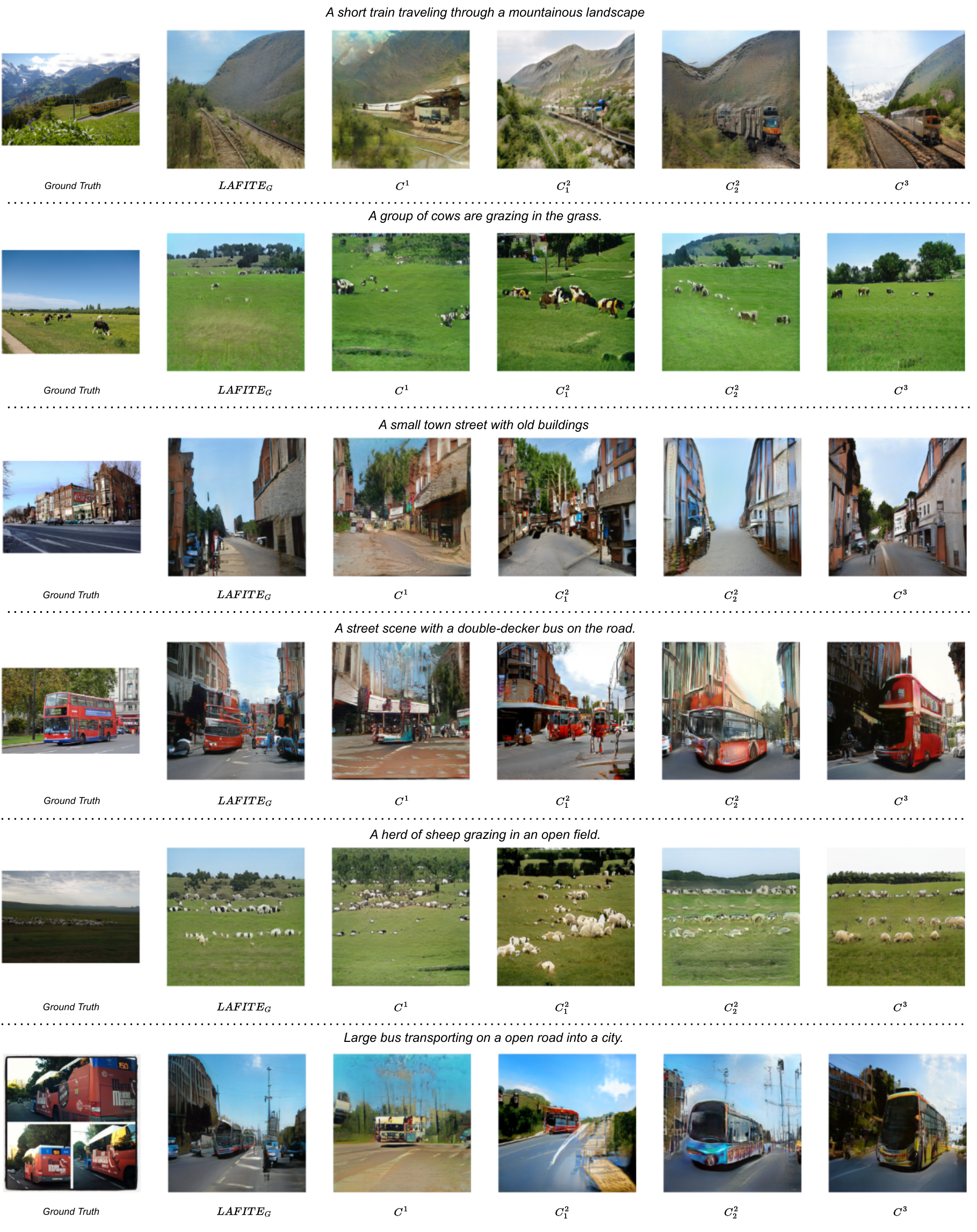}
    \caption{\emph{Qualitative examples of text-to-image generation on the MS-COCO dataset.} $C^1$ generates the basic scene that pertains to the text description, but are less photo-realistic in terms of contrast and color. $C^2_1$ and $C^2_2$ add more fine-grained detail despite still generating some artifacts. $C^3$ generates sharper images that are more detailed with fewer artifacts.}
    \label{fig:t2i}
\end{figure*}

We provide qualitative results for image generation in Figure~\ref{fig:t2i}, which helps us better understand the improvements of each component of $C^3$. We observe that:
\begin{itemize}
    \item $C^1$ generates the \textbf{basic scene} that pertains to the text description, but are less photo-realistic in terms of contrast and color.
    \item $C^2_1$ and $C^2_2$ add more \textbf{fine-grained detail} despite still generating some artifacts.
    \item $C^3$ generates \textbf{sharper images} that are more \textbf{detailed with fewer artifacts}.
\end{itemize}

\section{Why Align Embedding from Different Modalities?}
\label{sec:appendix_align}

If embeddings from different modalities are aligned, we can train a model on one modality and then infer on another modality, enabling us to build cross-modal applications with only uni-modal data. This is an emerging field that lacks principled approaches to be easily applied without requiring more empirical tuning.

We added an experiment to explain this further. We train a text generator (image captioner) over CLIP's text embeddings $x$. During inference, we manually shift all the $x$ to $x+c$ to simulate the modality gap (a constant vector orthogonal to original spans). We report the captioning performance in terms of gap distance $\|c\|$ in Table~\ref{tab:align_embeddings}. We observe substantial performance drops when the gap grows, showing the need to align embeddings.

\begin{table}[htbp]
\small
\centering
\begin{tabular}{c|ccc}
\toprule
\textbf{Gap Distance $\|c\|$} &	\textbf{ROUGE-1} &	\textbf{ROUGE-L}	& \textbf{METEOR} \\
\midrule
0.0	& 85.5	& 81.2	& 83.1 \\
0.2	& 76.3	& 71.5	& 73.8 \\
0.4	& 55.8	& 50.6	& 52.5 \\
0.6	& 40.6	& 35.9	& 36.6 \\
0.8	& 30.5	& 26.7	& 26.5 \\
1.0	& 24.1	& 21.0	& 20.3 \\
1.5	& 16.6	& 14.4	& 13.8 \\
2.0	& 13.8	& 12.1	& 11.6 \\
\bottomrule
\end{tabular}
\caption{\emph{Image captioning performance when trained on embedding $x$ and tested on $x+c$.} This shows the necessity to align embeddings when training a model on one modality and then inferring on another modality.}
\label{tab:align_embeddings}
\end{table}

\section{Effectiveness of Collapse vs Corrupt}
\label{sec:appendix_collapse_vs_corrupt}

In Table~\ref{tab:i2t}, we observe that adding noise (i.e., corrupt, $C_2^2$) is more effective than closing the modality gap (i.e., collapse, $C_2^1$). We hypothesize that the greater effectiveness of $C_2^2$ is because $C_2^2$ has two effects: 1) injecting noise in the span to mitigate alignment noise; 2) injecting noise in the modality gap direction to mitigate the model's sensitivity to this gap.

We have added an experiment to verify this hypothesis. When adding noise sampled from Gaussian distributions, we remove its component in the modality gap direction. Specifically, given $\epsilon \sim \mathcal{N}(0, \sigma^2 I)$, we compute its projection on the gap direction as $\epsilon_g = \frac{\epsilon \cdot g}{\|g\|} \frac{g}{\|g\|}$, where $\frac{g}{\|g\|}$ is the modality gap direction, then we remove this projection to get a new noise $\epsilon' = \epsilon - \epsilon_g$. We add this new noise during training and name this experiment as $C^2_2$ (span noise only).

From Table~\ref{tab:noise-effectiveness}, we see that adding noise only in the span (i.e., $C^2_2$ (span noise only)) makes its performance much worse than adding noise to all the directions (i.e., $C^2_2$), and its performance is similar to removing the modality gap (i.e., $C^2_1$). Therefore, adding noise (i.e., corrupt) actually leads to a similar improvement to removing the modality gap (i.e., collapse). The reason for the greater effectiveness of corrupt than collapse is that injecting Gaussian noise not only adds noise in the span but also to the modality gap direction.

Given the substantial size of the modality gap, adding noise is not enough to fully diminish the gap. Therefore, adding noise and removing the gap (i.e., $C^3$) still enhance the overall performance.

\begin{table*}[htbp]
\centering
\small
\setlength\tabcolsep{2pt}
\begin{tabular}{l|ccc|cccccc}
\toprule
\textbf{Method} & \textbf{Conn.} & \textbf{Coll.} & \textbf{Corr.} & \textbf{BLEU-1}$_\uparrow$  & \textbf{BLEU-4}$_\uparrow$  & \textbf{METEOR}$_\uparrow$  & \textbf{ROUGE-L}$_\uparrow$ & \textbf{CIDEr}$_\uparrow$  & \textbf{SPICE}$_\uparrow$  \\
\midrule
$C^1$ & \ding{51} & \ding{55} & \ding{55} & 28.1 & 2.4 & 12.2 & 25.4 & 13.0 & 6.8 \\
$C^2_1$ & \ding{51} & \ding{51} & \ding{55} & 44.4 & 6.1 & 15.5 & 33.6 & 25.2 & 9.2 \\ 
$C^2_2$ (span noise only) &	\ding{51} & \ding{55} & \ding{51} & 41.2 & 6.2 & 14.9 & 33.6 & 22.8 & 8.3 \\
$C^2_2$ & \ding{51} & \ding{55} & \ding{51} & 69.0 & 25.5 & 24.3 & 50.8 & 87.6 & 17.6 \\
$C^3$ & \ding{51} & \ding{51} & \ding{51} & \textbf{71.0} & \textbf{27.7} & 25.0 & \textbf{52.0} & \textbf{93.3} & \textbf{18.3} \\
\bottomrule
\end{tabular}
\caption{\emph{Collapse is actually as effective as corrupt.} The reason for the greater effectiveness of corrupt than collapse is that injecting Gaussian noise not only adds noise in the span but also to the modality gap direction.} 
\label{tab:noise-effectiveness}
\end{table*}

\section{Dimensional Collapse}
\label{sec:appendix_collapse}

Echoing our main paper, the phenomenon of dimensional collapse~\citep{jing2022understanding} in randomly initialized image and text encoders creates a modality gap prior to optimization that also persists after optimization. In this section, we offer further insights into dimensional collapse and demonstrate that 1) it is an inherent characteristic of deep neural networks and 2) deeper networks experience more pronounced dimensional collapse. Additionally, we establish a connection between dimensional collapse and the cone effect identified by~\citet{liang2022mind}.

\subsection{Real Networks}

In Figure~\ref{fig:svd2}, we find that the vision encoder (vision transformer) and text encoder (transformer) of the CLIP have the dimensional collapse phenomenon at initialization. Specifically, the effective dimension of the embeddings generated by the vision and text encoder is both much smaller than the full dimension, which induces a modality gap. Since the gradients of multi-modal contrastive learning will only be propagated in the effective dimensions, the modality gap will be preserved after optimization.

\begin{figure}[htbp]
    \centering
    \includegraphics[width=0.49\linewidth]{figures/collapse_random.png}
    \includegraphics[width=0.49\linewidth]{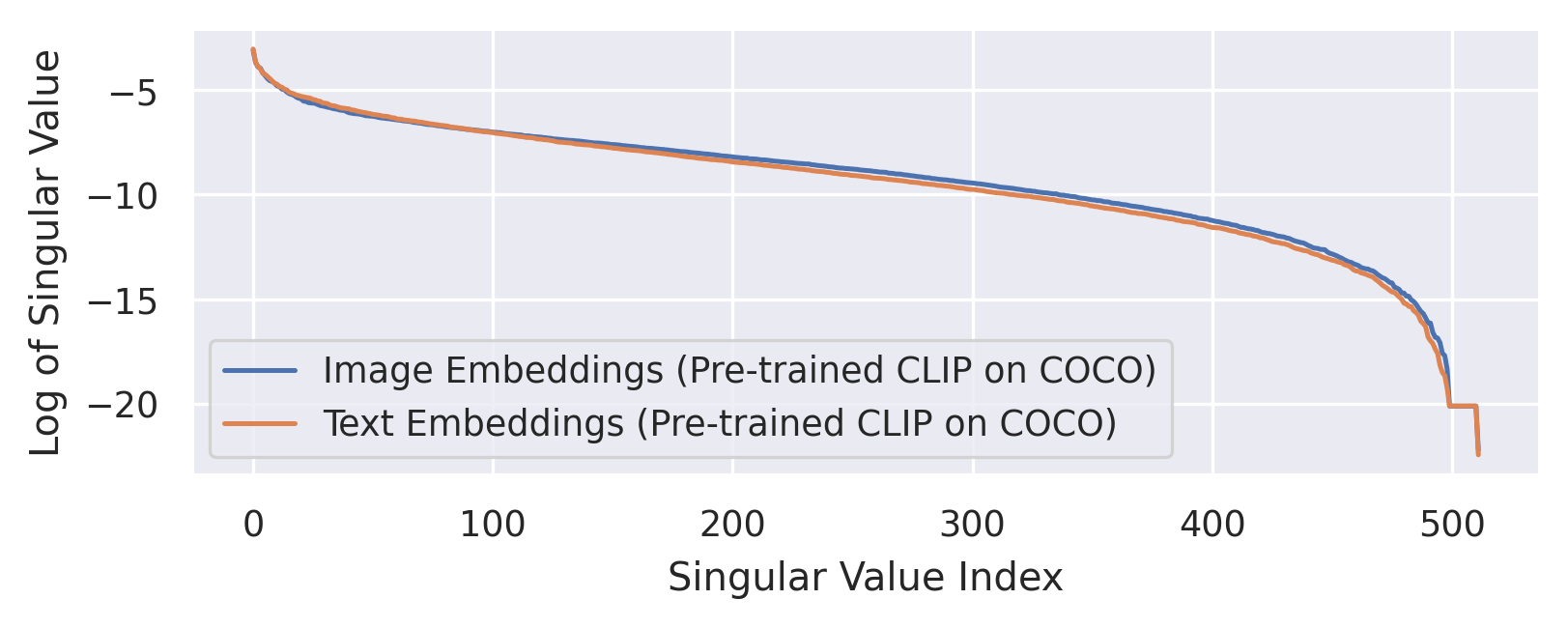}
    \caption{\emph{Dimensional collapse of the randomly initialized (left) and pre-trained (right) CLIP representation space.} Singular values obtained from SVD reveal that the effective dimension of the image and text representation space is much smaller than the total number of dimensions. }
    \label{fig:svd2}
\end{figure}

\subsection{Simulation}

To investigate the dimensional collapse phenomenon, we initialize a simple Multi-Layer Perceptron (MLP) with $n$ blocks, where each block contains a linear layer and Rectified Linear Unit (ReLU) activation. We set the dimensionality of the input space and hidden states to 512, and initialize the weights with Xavier uniform distribution and biases to zero. We initialize $N=1000$ input embeddings, where each dimension is sampled from a standard normal distribution $\mathcal{N}(0,1)$. We feed these embeddings into the MLP and extract the features after every 5 layers. We perform Singular Value Decomposition (SVD) on the feature covariance matrix to quantify the degree of dimensional collapse.

As shown in Figure~\ref{fig:dim_collapse_simulation}, our experiments reveal two key insights:
\begin{enumerate}
\item Dimensional collapse is an inherent characteristic of deep neural networks and even a simple MLP exhibits this behavior.
\item Deeper networks are more prone to dimensional collapse, resulting in a smaller effective dimension of the feature space.
\end{enumerate}

\begin{figure}[htbp]
    \centering
    \includegraphics[width=0.5\linewidth]{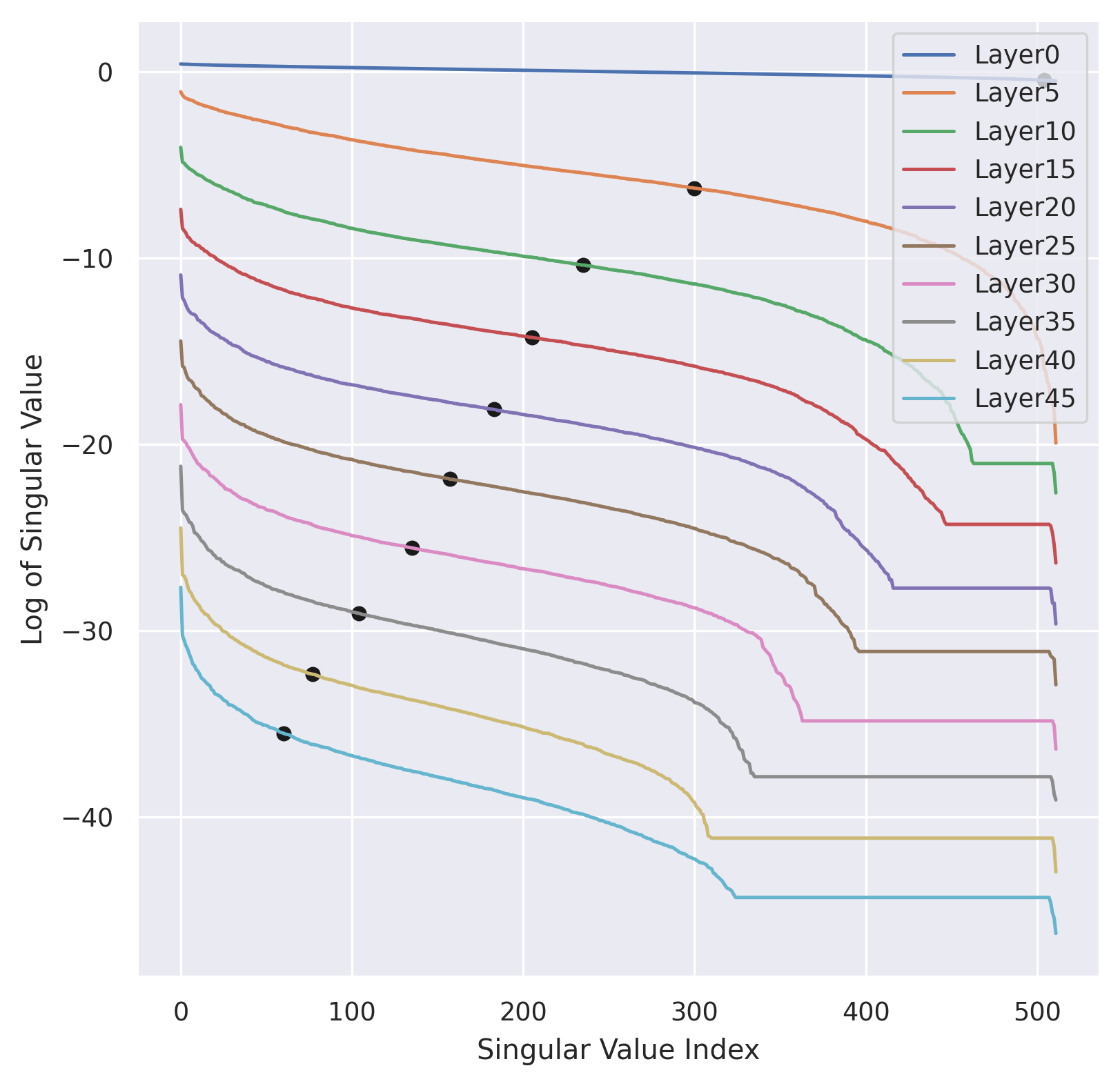}
    \vspace{-1em}
    \caption{\emph{Simulation of dimensional collapse on a MLP ($n$*(ReLU(Linear))) network.} The $x$-coordinates of the black dots indicate the effective dimensions for the embeddings from each of the layers, respectively, which quantifies the extent of collapse. }
    \label{fig:dim_collapse_simulation}
\end{figure}

\subsection{Connection to Cone Effect}

The dimensional collapse phenomenon can provide an explanation for the cone effect observed by~\citet{liang2022mind}. They found that the cosine similarities between any two embeddings outputted by a deep neural network were significantly higher than zero. Due to the dimensional collapse, all embeddings share similar values in the ineffective dimensions, leading to high cosine similarities. The stronger the dimensional collapse, the greater the number of ineffective dimensions, and hence, a stronger cone effect. This explanation is illustrated in Figure~\ref{fig:cone}, where collapsing the $z$-axis of a 3D representation space induces a cone, leading to significantly higher cosine similarities between any two embeddings.

\begin{figure}[htbp]
    \centering
    \includegraphics[width=0.3\linewidth]{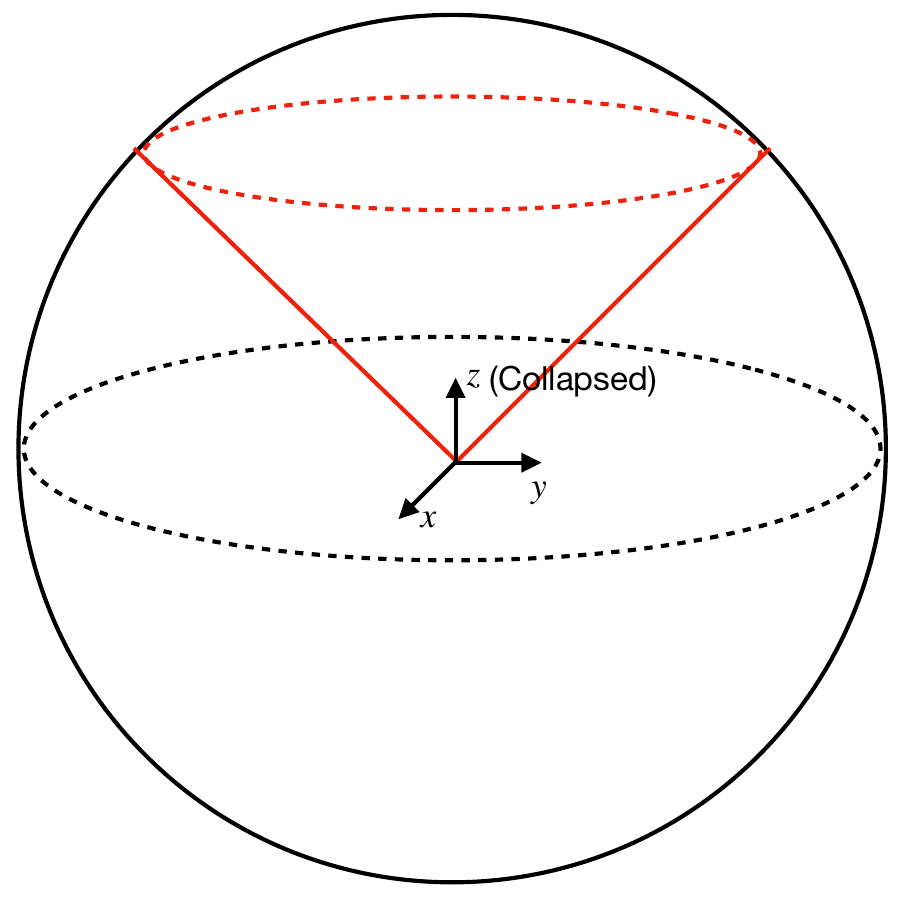}
    \caption{\emph{Dimensional collapse explains the cone effect of deep neural networks.} When the $z$-axis of a 3D representation space is collapsed, it results in a cone shape, where the cosine similarities between any two embeddings are significantly higher than zero.}
    \label{fig:cone}
\end{figure}

\end{document}